\documentclass{article}

\usepackage[preprint,nonatbib]{neurips_2021}

\usepackage[utf8]{inputenc} %
\usepackage[T1]{fontenc}    %
\usepackage{hyperref}       %
\usepackage{url}            %
\usepackage{booktabs}       %
\usepackage{amsfonts}       %
\usepackage{nicefrac}       %
\usepackage{microtype}      %
\usepackage{xcolor}         %
\usepackage{algorithm}
\usepackage{algorithmicx}
\usepackage{wrapfig}
\usepackage{booktabs}
\usepackage{algpseudocode}
\usepackage{diagbox}
\usepackage{comment}
\usepackage{multirow}
\usepackage{subcaption}
\usepackage{caption}
\usepackage{wrapfig}
\usepackage{graphicx}
\usepackage{float}
\usepackage{thm-restate}
\hypersetup{colorlinks = true, linkcolor = brown,
            urlcolor  = gray,
            citecolor = brown,
            anchorcolor = brown}
\usepackage[framemethod=tikz]{mdframed}
\usepackage{lipsum}

\definecolor{mycolor}{rgb}{0.122, 0.435, 0.698}
\usepackage[most]{tcolorbox}

\usepackage{amsmath,amsfonts,bm,amssymb,mathtools,amsthm}
\usepackage{color,xcolor,xspace}
\usepackage{booktabs}
\usepackage{thm-restate}
\usepackage{bbding}
\usepackage{pifont}
\usepackage{wasysym}

\newtheorem{theorem}{Theorem}
\newtheorem{definition}{Definition}

\newcommand{\bb}[1]{{\mathbb{#1}}}
\newcommand{\norm}[1]{{\lVert {#1} \rVert}}
\newcommand{\diff}{\mathrm{d}}

\def\eqref#1{Eq.(\ref{#1})}
\def\Eqref#1{Equation~(\ref{#1})}

\def\1{\bm{1}}

\def\rvc{{\mathbf{c}}}

\def\rvn{{\mathbf{n}}}

\def\rvw{{\mathbf{w}}}
\def\rvx{{\mathbf{x}}}
\def\rvy{{\mathbf{y}}}
\def\rvz{{\mathbf{z}}}

\def\vz{{\bm{z}}}

\def\mI{{\bm{I}}}

\DeclareMathAlphabet{\mathsfit}{\encodingdefault}{\sfdefault}{m}{sl}
\SetMathAlphabet{\mathsfit}{bold}{\encodingdefault}{\sfdefault}{bx}{n}

\def\gC{{\mathcal{C}}}

\def\gN{{\mathcal{N}}}

\def\gP{{\mathcal{P}}}

\def\gS{{\mathcal{S}}}

\def\gX{{\mathcal{X}}}
\def\gY{{\mathcal{Y}}}
\def\gZ{{\mathcal{Z}}}

\newcommand{\pdata}{p_{\rm{data}}}

\newcommand{\E}{\mathbb{E}}

\newcommand{\R}{\mathbb{R}}

\newcommand{\KL}{D_{\mathrm{KL}}}

\newcommand{\supp}{\mathrm{supp}}

\newcommand{\ptheta}{p_\theta}
\newcommand{\qphi}{q_\phi}

\title{D2C: Diffusion-Decoding Models for \\ Few-Shot Conditional Generation}

\author{%
  Abhishek Sinha\thanks{Equal contribution.} \\
  \normalfont {Department of Computer Science}\\
  \normalfont {Stanford University}\\
  \texttt {a7b23@stanford.edu}
  \And 
  Jiaming Song\footnotemark[1] \\
  \normalfont {Department of Computer Science}\\
  \normalfont {Stanford University}\\
  \texttt {tsong@cs.stanford.edu}
  \AND
  Chenlin Meng \\
  \normalfont {Department of Computer Science}\\
  \normalfont {Stanford University}\\
  \texttt {chenlin@cs.stanford.edu}
  \And
  Stefano Ermon 
  \\
  \normalfont {Department of Computer Science}\\
  \normalfont {Stanford University}\\
  \texttt {ermon@stanford.edu}
  \\
  \\
  }

\begin{document}

\maketitle

\newcommand{\pname}{Diffusion-Decoding generative models with Contrastive representations}

\begin{abstract}
Conditional generative models of high-dimensional images have many applications, but supervision signals from conditions to images can be expensive to acquire.
This paper describes Diffusion-Decoding models with Contrastive representations (D2C), a paradigm for training unconditional variational autoencoders (VAEs) for few-shot conditional image generation.
D2C uses a learned diffusion-based prior over the latent representations to improve generation and contrastive self-supervised learning to improve representation quality. 
D2C can adapt to novel generation tasks conditioned on labels or manipulation constraints, by learning from as few as 100 labeled examples. On conditional generation from new labels, D2C achieves superior performance over state-of-the-art VAEs and diffusion models. On conditional image manipulation, D2C generations are two orders of magnitude faster to produce over StyleGAN2 ones and are preferred by $50\% - 60\%$ of the human evaluators in a double-blind study.
\end{abstract}
\section{Introduction}
Generative models trained on large amounts of unlabeled data have achieved great success in various domains including images~\cite{brock2018large,karras2020analyzing,razavi2019generating,ho2020denoising}, text~\cite{li2020optimus,aneja2019sequential}, audio~\cite{dhariwal2020jukebox,ping2020waveflow,oord2016wavenet,mittal2021symbolic}, and graphs~\cite{grover2019graphite,niu2020permutation}.
However, downstream applications of generative models 
are often based on various conditioning signals, %
such as labels~\cite{mirza2014conditional}, text descriptions~\cite{mansimov2015generating}, reward values~\cite{you2018graph}, or similarity with existing data~\cite{isola2017image}. While it is possible to directly train conditional models, this often requires large amounts of paired data~\cite{lin2014microsoft,ramesh2021zero} that are costly to acquire. Hence, it would be desirable to learn 
conditional generative models using large amounts of unlabeled data and as little paired data as possible.

Contrastive self-supervised learning (SSL) methods can greatly reduce the need for labeled data in discriminative tasks by learning effective representations from unlabeled data~\cite{oord2018representation,he2019momentum,grill2020bootstrap}, and have also been shown to improve few-shot learning~\cite{henaff2020data}. It is therefore natural to ask if they can also be used to improve few-shot generation. Latent variable generative models (LVGM) are a natural candidate for this, since they already involve a low-dimensional, structured latent representation of the data they generate. 
However, popular LVGMs, such as generative adversarial networks (GANs,~\cite{goodfellow2014generative,karras2020analyzing}) and diffusion models~\cite{ho2020denoising,song2020denoising}, lack explicit 
tractable 
functions to map inputs to representations, making it difficult to optimize latent variables with SSL. %
Variational autoencoders (VAEs,~\cite{kingma2013auto,rezende2015variational}), on the other hand, can naturally adopt SSL through their encoder model, but they typically have worse sample quality.

\begin{figure}[H]
    \centering
    \includegraphics[width=\textwidth]{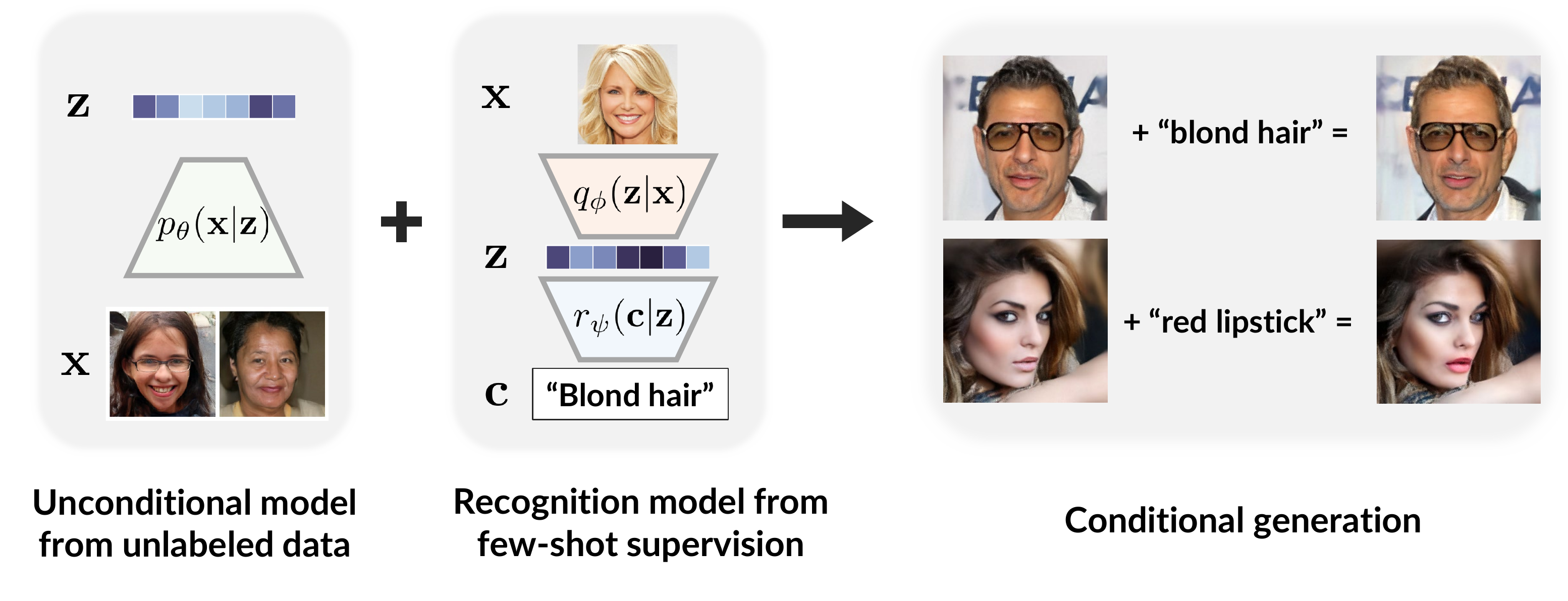}
    \caption{Few-shot conditional generation with the unconditional D2C model (left). With a recognition model over the latent space (middle), D2C can generate samples for novel conditions, such as image manipulation (right). These conditions can be defined with very few labels.}
    \label{fig:few-shot-intro}
\end{figure}

In this paper, we propose Diffusion-Decoding models with Contrastive representations (D2C), a special VAE that is suitable for conditional few-shot generation.
D2C uses contrastive self-supervised learning methods to obtain a latent space that inherits the transferrability and few-shot capabilities of self-supervised representations. 
Unlike other VAEs, D2C learns a diffusion model over the latent representations. 
This latent diffusion model ensures that D2C uses the same latent distribution for both training and generation. 
We provide a formal argument to explain why this approach may lead to better sample quality than existing  hierarchical VAEs. 
We further discuss how to apply D2C to few-shot conditional generation where the conditions are defined through labeled examples and/or manipulation constraints. Our approach combines a discriminative model providing conditioning signal and generative diffusion model over the latent space, and is computationally more efficient than methods that act directly over the image space (Figure~\ref{fig:few-shot-intro}).

We evaluate and compare D2C with several state-of-the-art generative models over 6 datasets. On unconditional generation, D2C outperforms state-of-the-art VAEs and is competitive with diffusion models under similar computational budgets. On conditional generation with 100 labeled examples, D2C significantly outperforms state-of-the-art VAE~\cite{vahdat2020nvae} and diffusion models~\cite{song2020denoising}. D2C can also learn to perform certain image manipulation tasks from as few as 100 labeled examples. Notably, for manipulating images, %
D2C is two orders of magnitude faster than StyleGAN2~\cite{zhu2020domain} and preferred by $50\% - 60\%$ of human evaluations, which to our best knowledge is the first for any VAE model.

\section{Background}

\paragraph{Latent variable generative models} A latent variable generative model (LVGM) is posed as a conditional distribution $\ptheta: \gZ \to \gP(\gX)$ from a latent variable $\rvz$ to a generated sample $\rvx$, parametrized by $\theta$. To acquire new samples, LVGMs draw random latent variables $\rvz$ from some distribution $p(\rvz)$ and map them to image samples through $\rvx \sim \ptheta(\rvx | \rvz)$. %
Most LVGMs are built on top of four paradigms: variational autoencoders (VAEs,~\cite{kingma2013auto,rezende2015variational}), Normalizing Flows (NFs,~\cite{dinh2014nice,dinh2016density}), Generative Adversarial Networks (GANs,~\cite{goodfellow2014generative}), and diffusion / score-based generative models~\cite{ho2020denoising,song2019generative}.%

Particularly, VAEs use an inference model from $\rvx$ to $\rvz$ for training. Denoting the inference distribution from $\rvx$ to $\rvz$ as $\qphi(\rvz | \rvx)$, the generative distribution from $\rvz$ to $\rvx$ as $\ptheta(\rvx | \rvz)$, VAEs are trained by minimizing the following upper bound of negative log-likelihood:
\begin{align}
    L_{\text{VAE}} = \E_{\rvx \sim \pdata}[\E_{\rvz \sim \qphi(\rvz | \rvx)}[-\log p(\rvx | \rvz)] + \KL(\qphi(\rvz | \rvx) \Vert p(\rvz))]
\end{align}
where $\pdata$ is the data distribution and $\KL$ is the KL-divergence.

\paragraph{Diffusion models}
Diffusion models~\cite{sohl-dickstein2015deep,ho2020denoising,song2020denoising} produce samples by reversing a Gaussian diffusion process. We use the index $\alpha \in [0, 1]$ to denote the particular noise level of an noisy observation $\rvx^{(\alpha)} = \sqrt{\alpha} \rvx + \sqrt{1 - \alpha} \epsilon$, where $\rvx$ is the clean observation and $\epsilon \sim \gN(0, \mI)$ is a standard Gaussian distribution; as $\alpha \to 0$, the distribution of $\rvx^{(\alpha)}$ converges to $\gN(0, \mI)$. Diffusion models are typically parametrized as reverse noise models $\epsilon_\theta(\rvx^{(\alpha)}, \alpha)$ that predict the noise component of $\rvx^{(\alpha)}$ given a noise level $\alpha$, and trained to minimize $\Vert \epsilon_\theta(\rvx^{(\alpha)}, \alpha) - \epsilon \Vert_2^2$, the mean squared error loss between the true noise and predicted noise. Given any non-increasing series $\{\alpha_i\}_{i=0}^{T}$ between 0 and 1, the diffusion objective for a clean sample from the data $\rvx$ is:
\begin{align}
    \ell_{\text{diff}}(\rvx; w, \theta) := \sum_{i=1}^{T} w(\alpha_i) \bb{E}_{\epsilon \sim \gN(0, \mI)}[\norm{\epsilon - \epsilon_{\theta}(\rvx^{(\alpha_i)}, \alpha_i)}_2^2] , \quad \rvx^{(\alpha_i)} := \sqrt{\alpha_i} \rvx + \sqrt{1 - \alpha_i} \epsilon \label{eq:diffusion-training-obj}
\end{align}
where $w: \{\alpha_i\}_{i=1}^{T} \to \R_{+}$ controls the loss weights for each $\alpha$. When $w(\alpha) = 1$ for all $\alpha$, we recover the denoising score matching objective for training score-based generative models \cite{song2019generative}.

Given an initial sample $\rvx_0 \sim \gN(0, \mI)$, diffusion models acquires clean samples (\textit{i.e.}, samples of $\rvx_1$) through a gradual denoising process, where samples with reducing noise levels $\alpha$ are produced (\textit{e.g.}, $\rvx_0 \to \rvx_{0.3} \to \rvx_{0.7} \to \rvx_1$). 
In particular, Denoising Diffusion Implicit Models (DDIMs, \cite{song2020denoising}) uses an Euler discretization of some neural ODE~\cite{chen2018neural} to produce samples (Figure~\ref{fig:d2}, left). 

We provide a more detailed description for training diffusion models in Appendix~\ref{app:diffusion-training} and sampling from DDIM in Appendix~\ref{app:ddim}. For conciseness, we use the notation $p^{(\alpha)}(\rvx^{(\alpha)})$ to denote the marginal distribution of $\rvx^{(\alpha)}$ under the diffusion model, and $p^{(\alpha_1, \alpha_2)}(\rvx^{(\alpha_2)}$ | $\rvx^{(\alpha_1)})$ to denote the diffusion sampling process from $\rvx^{(\alpha_1)}$ to $\rvx^{(\alpha_2)}$ (assuming $\alpha_1 < \alpha_2$). This notation abstracts away the exact sampling procedure of the diffusion model, which depends on choices of $\alpha$. %

\paragraph{Self-supervised learning of representations}
In self-supervised learning (SSL), representations are learned by completing certain pretext tasks that do not require extra manual labeling~\cite{noroozi2016unsupervised,devlin2018bert}; these representations can then be applied to other downstream tasks, often in few-shot or zero-shot scenarios.
In particular, contrastive representation learning encourages representations to be closer between ``positive'' pairs and further between ``negative'' pairs; contrastive predictive coding (CPC, \cite{oord2018representation}), based on multi-class classification, have been commonly used in state-of-the-art methods~\cite{he2019momentum,chen2020improved,chen2021an,chen2020a,song2020multi}. 
Other non-contrastive methods exist, such as BYOL~\cite{grill2020bootstrap} and SimSiam~\cite{chen2020exploring}, but they usually require additional care to prevent the representation network from collapsing.

\section{Problem Statement}
\paragraph{Few-shot conditional generation} %
Our goal is to learn an unconditional generative model $\ptheta(\rvx)$ such that it is suitable for conditional generation. Let $\gC(\rvx, \rvc, f)$ describe an event that ``$f(\rvx) = \rvc$'', where  $\rvc$ is a property value and $f(\rvx)$ is a property function that is unknown at training. In conditional generation, our goal is to sample $\rvx$ such that the event $\gC(\rvx, \rvc, f)$ occurs for a chosen $\rvc$. If we have access to some ``ground-truth'' model that gives us $p(\gC | \rvx) := p(f(\rvx) = \rvc | \rvx)$, then the conditional model can be derived from Bayes' rule:
$
    \ptheta(\rvx | \gC) \propto p(\gC | \rvx) \ptheta(\rvx)
$.
These properties $\rvc$ include (but are not limited to\footnote{When $\gC$ refers to an event that is always true, we recover unconditioned generation.}) labels~\cite{mirza2014conditional}, text descriptions~\cite{mansimov2015generating,reed2016generative}, noisy or partial observations~\cite{candes2006stable,asim2019invertible,kadkhodaie2020solving,daras2021intermediate}, and manipulation constraints~\cite{park2020swapping}. In many cases, we do not have direct access to the true $f(\rvx)$, so we need to learn an accurate model from labeled data~\cite{bartunov2018few} (\textit{e.g.}, $(\rvc, \rvx)$ pairs). %

\paragraph{Desiderata} 
Many existing methods are optimized for some known condition (\textit{e.g.}, labels in conditional GANs~\cite{brock2018large}) or assume abundant pairs between images and conditions that can be used for pretraining (\textit{e.g.}, DALL-E~\cite{ramesh2021zero} and CLIP~\cite{radford2021learning} over image-text pairs). Neither is the case in this paper, as we do not expect to train over paired data. %

While high-quality latent representations are not essential to unconditional image generation (\textit{e.g.}, autoregressive~\cite{oord2016pixel}, energy-based~\cite{du2019implicit}, and some diffusion models~\cite{ho2020denoising}), they can be beneficial when we wish to specify certain conditions with limited supervision signals, similar to how SSL representations can reduce labeling efforts in downstream tasks. A compelling use case is detecting and removing biases in datasets via image manipulation, where we should not only address well-known biases a-priori but also address other hard-to-anticipate biases, adapting to societal needs~\cite{najibi2020racial}. 

Therefore, a desirable generative model should not only have high sample quality but also contain informative latent representations. While VAEs are ideal for learning rich latent representations due to being able to incorporate SSL within the encoder, they generally do not achieve the same level of sample quality as GANs and diffusion models.

\newcommand{\CG}{{\color{teal}\ding{52}}}
\newcommand{\XR}{{\color{purple}\ding{55}}}

\begin{table}[]
    \centering
    \caption{
    A comparison of several common paradigms for generative modeling. [Explicit $\rvx \to \rvz$]: the mapping from $\rvx$ to $\rvz$ is directly trainable, which enables SSL; [No prior hole]: latent distributions used for generation and training are identical (Sec.~\ref{sec:prior-hole}), which improves generation; [Non-adversarial]: training procedure does not involve adversarial optimization, which improves training stability.}
    \vspace{.5em}
    \begin{tabular}{l|ccc}
    \toprule
      \multirow{2}{*}{\textbf{Model family}}   & Explicit $\rvx \to \rvz$ &  No prior hole & Non-Adversarial \\
         & (Enables SSL) & (Better generation) & (Stable training) \\\midrule
        VAE~\cite{kingma2013auto,rezende2015variational}, NF~\cite{dinh2014nice} & \CG & \XR & \CG \\
        GAN~\cite{goodfellow2014generative} & \XR & \CG & \XR \\
        BiGAN~\cite{donahue2016adversarial,dumoulin2016adversarially} & \CG & \CG & \XR \\
        DDIM~\cite{song2020denoising} & \XR & \CG & \CG \\\midrule
        \textbf{D2C} & \CG & \CG & \CG \\\bottomrule
    \end{tabular}
    
    \label{tab:summary}
\end{table}

\section{Diffusion-Decoding Generative Models with Contrastive Learning}
To address the above issue, we present Diffusion-Decoding generative models with Contrastive Learning (D2C), an extension to VAEs with high-quality samples and high-quality latent representations, and are thus well suited to few-shot conditional generation. Moreover, unlike GAN-based methods, D2C does not involve unstable adversarial training (Table~\ref{tab:summary}). 

As its name suggests, the generative model for D2C has two components -- \textit{diffusion} and \textit{decoding}; the \textit{diffusion} component operates over the latent space and the \textit{decoding} component maps from latent representations to images. Let us use the $\alpha$ index notation for diffusion random variables: $\rvz^{(0)} \sim p^{(0)}(\rvz^{(0)}) := \gN(0, \mI)$ is the ``noisy'' latent variable with $\alpha = 0$, and $\rvz^{(1)}$ is the ``clean'' latent variable with $\alpha = 1$. The generative process of D2C, which we denote $\ptheta(\rvx | \rvz^{(0)})$, is then defined as:
\begin{align}
    \rvz^{(0)} \sim p^{(0)}(\rvz^{(0)}), \quad \rvz^{(1)} \sim \underbrace{\ptheta^{(0,1)}(\rvz^{(1)} | \rvz^{(0)})}_{\text{diffusion}}, \quad \rvx \sim \underbrace{\ptheta(\rvx | \rvz^{(1)})}_{\text{decoding}},
\end{align}
where $p^{(0)}(\rvz^{(0)}) = \gN(0, \mI)$ is the prior distribution for the diffusion model, $\ptheta^{(0,1)}(\rvz^{(1)} | \rvz^{(0)})$ is the diffusion process from $\rvz^{(0)}$ to $\rvz^{(1)}$, and $\ptheta(\rvx | \rvz^{(1)})$ is the decoder from $\rvz^{(1)}$ to $\rvx$. Intuitively, D2C models produce samples by drawing $\rvz^{(1)}$ from a diffusion process and then decoding $\rvx$ from $\rvz^{(1)}$. 

In order to train a D2C model, we use an inference model $\qphi(\rvz^{(1)} | \rvx)$ that predicts proper $\rvz^{(1)}$ latent variables from $\rvx$; this can directly incorporate SSL methods~\cite{xie2021adversarial}, leading to the following objective:
\begin{align}
    L_{\text{D2C}}(\theta, \phi; w) & := L_\mathrm{D2}(\theta, \phi; w) + \lambda L_{\mathrm{C}}(q_\phi), \label{eq:d2c-obj}\\
    L_\mathrm{D2}(\theta, \phi; w) & := \E_{\rvx \sim \pdata, \rvz^{(1)} \sim \qphi(\rvz^{(1)} | \rvx)}[-\log p(\rvx | \rvz^{(1)}) + \ell_{\text{diff}}(\rvz^{(1)}; w, \theta)], 
\end{align}
where $\ell_{\text{diff}}$ is defined as in \eqref{eq:diffusion-training-obj}, $L_C(q_\phi)$ denotes any contrastive predictive coding objective~\cite{oord2018representation} with rich data augmentations \cite{he2019momentum,chen2020improved,chen2021an,chen2020a,song2020multi} (details in Appendix \ref{app:cpc}) and $\lambda > 0$ is a weight hyperparameter. The first two terms, which we call $L_{\mathrm{D2}}$, contains a ``reconstruction loss'' ($-\log p(\rvx | \rvz^{(1)})$) and a ``diffusion loss'' over samples of $\rvz^{(1)} \sim \qphi(\rvz^{(1)} | \rvx)$. We illustrate the D2C generative and inference models in Figure~\ref{fig:d2}, and its training procedure in Appendix~\ref{app:d2c-training}.

\begin{figure}
    \centering
    \includegraphics[width=0.8\textwidth]{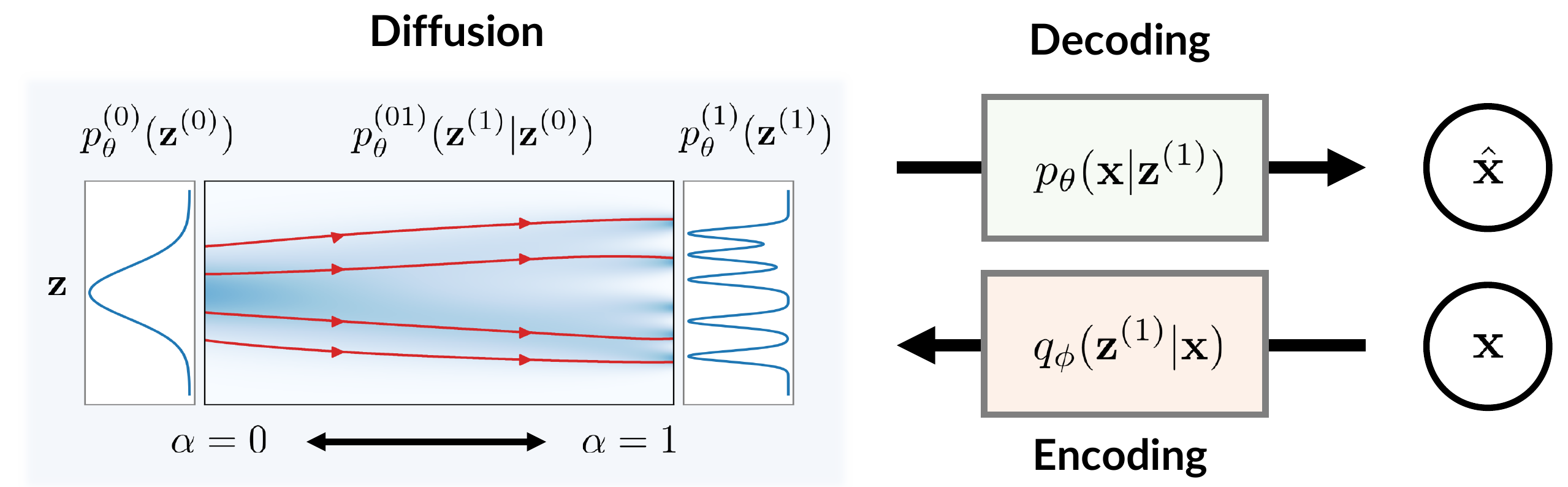}
    \caption{Illustration of components of a D2 model. On top of the encoding and decoding between $\rvx$ and $\rvz^{(1)}$, we use a diffusion model to generate $\rvz^{(1)}$ from a Gaussian $\rvz^{(0)}$. The red lines describe several smooth ODE trajectories from $\alpha = 0$ to $\alpha = 1$ corresponding to DDIM.}
    \label{fig:d2}
\end{figure}

\subsection{Relationship to maximum likelihood} 
The D2 objective ($L_{\mathrm{D2}}$) appears similar to the original VAE objective ($L_{\mathrm{VAE}}$). %
Here, we make an informal statement that the D2 objective function is deeply connected to the variational lower bound of log-likelihood; we present the full statement and proof in Appendix~\ref{app:mle}.

\begin{restatable}{theorem}{mleinf}(informal)
For any valid $\{\alpha_i\}_{i=0}^{T}$, there exists some weights  $\hat{w}: \{\alpha_i\}_{i=1}^{T} \to \R_{+}$ for the diffusion objective such that $-L_{\mathrm{D2}}$ is a variational lower bound to the log-likelihood, \textit{i.e.},
\begin{align}
    - L_{\mathrm{D2}}(\theta, \phi; \hat{w}) \leq  \bb{E}_{\pdata}[\log \ptheta(\rvx)],
\end{align}
where $\ptheta(\rvx) := \bb{E}_{\rvx_0 \sim p^{(0)}(\rvz^{(0)})}[\ptheta(\rvx | \rvz^{(0)})]$ is the marginal probability of $\rvx$ under the D2C model.
\end{restatable}
\begin{proof}(sketch)
The diffusion term $\ell_{\mathrm{diff}}$ upper bounds the KL divergence between $\qphi(\rvz_{1} | \rvx)$ and $p_\theta^{(1)}(\rvz^{(1)})$ for suitable weights~\cite{ho2020denoising,song2020denoising}, which recovers a VAE objective.
\end{proof}
\subsection{D2C models address latent posterior mismatch in VAEs} \label{sec:prior-hole}
While D2C is a special case of VAE, we argue that D2C is non-trivial in the sense that it addresses a long-standing problem in VAE methods~\cite{tomczak2017vae,takahashi2019variational}, namely the mismatch  between the prior distribution $\ptheta(\rvz)$ and the aggregate (approximate) posterior distribution $\qphi(\rvz) := \E_{\pdata(\rvx)}[\qphi(\rvz | \rvx)]$. A mismatch could create ``holes''~\cite{rosca2018distribution,hoffman2016elbo,aneja2020ncp} in the prior that the aggregate posterior fails to cover during training, resulting in worse sample quality, as many latent variables used during generation are likely to never have been trained on. We formalize this notion in the following definition. 
\begin{definition}[Prior hole] \label{def:prior-hole}
Let $p(\rvz), q(\rvz)$ be two distributions with $\supp(q) \subseteq \supp(p)$. We say that $q$ has an \textbf{$(\epsilon, \delta)$-prior hole} with respect to (the prior) $p$ for $\epsilon, \delta \in (0, 1)$, $\delta > \epsilon$, if there exists a set $S \in \supp(P)$, such that $\int_S p(\rvz) \diff \rvz \geq \delta$ and $\int_S q(\rvz) \diff \rvz \leq \epsilon$. 
\end{definition}
Intuitively, if $\qphi(\rvz)$ has a prior hole with large $\delta$ and small $\epsilon$ (\textit{e.g.}, inversely proportional to the number of training samples), then it is very likely that latent variables within the hole are never seen during training (small $\epsilon$), yet frequently used to produce samples (large $\delta$). Most existing methods address this problem by optimizing certain statistical divergences between $\qphi(\rvz)$ and $\ptheta(\rvz)$, such as the KL divergence or Wasserstein distance~\cite{tolstikhin2017wasserstein}. However, we argue in the following statement that prior holes might not be eliminated even if we optimize certain divergence values to be reasonably low, especially when $\qphi(\rvz)$ is very flexible. We present the formal statement and proof in Appendix~\ref{app:prior-hole}.
\begin{wrapfigure}[4]{r}{0.25\textwidth}
\centering
    \includegraphics[width=0.25\textwidth]{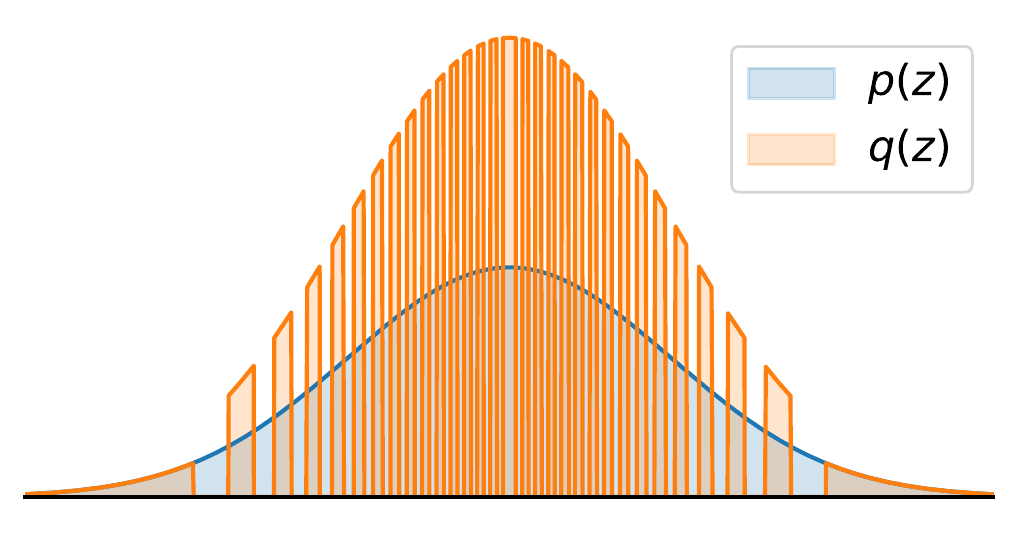}
\end{wrapfigure}
\begin{restatable}{theorem}{priorholeinf}(informal)
\label{thm:prior-hole-informal}
Let $\ptheta(\rvz) = \gN(0, 1)$. For any $\epsilon > 0$, there exists a distribution $\qphi(\rvz)$ with an $(\epsilon, 0.49)$-prior hole, such that $\KL(\qphi \Vert \ptheta) \leq \log 2$\footnote{This is reasonably low for realistic VAE models (NVAE~\cite{vahdat2020nvae} reports a KL divergence of around 2810 nats).}%
and $W_2(\qphi, \ptheta) < \gamma$ for any $\gamma > 0$, where $W_2$ is the 2-Wasserstein distance. 
\end{restatable}
\vspace{-1em}
\begin{proof}(sketch)
We construct a $\qphi$ that satisfies these properties (top-right figure). First, we truncate the Gaussian and divide them into regions with same probability mass; then we support $\qphi$ over half of these regions (so $\delta > 0.49$); finally, we show that the divergences are small enough.
\end{proof}
In contrast to addressing prior holes by optimization, diffusion models eliminate prior holes by construction, since the diffusion process from $\rvz^{(1)}$ to $\rvz^{(0)}$ is constructed such that the distribution of $\rvz^{(\alpha)}$ always converges to a standard Gaussian as $\alpha \to 0$. As a result, the distribution of latent variables used during training is arbitrarily close to that used in generation\footnote{We expand this argument in Appendix~\ref{app:prior-hole}.}, which is also the case in GANs. %
Therefore, our argument provides an explanation as to why we observe better sample quality results from GANs and diffusion models than VAEs and NFs.

\section{Few-shot Conditional Generation with D2C}

In this section, we discuss how D2C can be used to learn to perform conditional generation from few-shot supervision. We note that D2C is only trained on images and not with any other data modalities (\textit{e.g.}, image-text pairs~\cite{ramesh2021zero}) or supervision techniques (\textit{e.g.}, meta-learning~\cite{clouatre2019figr,bartunov2018few}).

\begin{wrapfigure}[9]{R}{0.5\textwidth}
\centering
\vspace{-1.5em}
\begin{minipage}{0.5\textwidth}
\begin{algorithm}[H]
\small
    \caption{Conditional generation with D2C}
    \label{alg:few-shot-examples}
    \begin{algorithmic}[1]
    \State \textbf{Input} $n$ examples $\{(\rvx_i, \rvc_i)\}_{i=1}^{n}$, property $\rvc$.
    \State Acquire latents $\rvz^{(1)}_i \sim \qphi(\rvz^{(1)} | \rvx_i)$ for $i \in [n]$;
    \State Train model $r_\psi(\rvc | \rvz^{(1)})$ over $\{(\rvz^{(1)}_i, \rvc_i)\}_{i=1}^{n}$
    \State Sample latents with $\hat{\rvz}^{(1)} \sim r_\psi(\rvc | \rvz^{(1)}) \cdot p_\theta^{(1)}(\rvz^{(1)})$ (unnormalized);
    \State Decode $\hat{\rvx} \sim \ptheta(\rvx | \hat{\rvz}^{(1)})$.
    \State \textbf{Output} $\hat{\rvx}$.
    \end{algorithmic}
\end{algorithm}
\end{minipage} 
\end{wrapfigure}
\paragraph{Algorithm} We describe the general algorithm for conditional generation from a few images in Algorithm~\ref{alg:few-shot-examples}, and detailed implementations in Appendix~\ref{app:exp}. With a model over the latent space (denoted as $r_\psi(\rvc | \rvz^{(1)})$), we draw conditional latents from an unnormalized distribution with the diffusion prior (line 4). 
This can be implemented in many ways such as rejection sampling or Langevin dynamics~\cite{nguyen2017plug,song2020score,dhariwal2021diffusion}.

\paragraph{Conditions from labeled examples} Given a few labeled examples, %
we wish to produce diverse samples with a certain label. For labeled examples we can directly train a classifier over the latent space, which we denote as $r_\psi(\rvc | \rvz^{(1)})$ with $\rvc$ being the class label and $\rvz^{(1)}$ being the latent representation of $\rvx$ from $q_\phi(\rvz^{(1)} | \rvx)$.
If these examples do not have labels (\textit{i.e.}, we merely want to generate new samples similar to given ones), we can train a positive-unlabeled (PU) classifier~\cite{elkan2008learning} where we assign ``positive'' to the new examples and ``unlabeled'' to training data. Then we use the classifier with the diffusion model $\ptheta(\rvz^{(1)} | \rvz^{(0)})$ to produce suitable values of $\rvz^{(1)}$, such as by rejecting samples from the diffusion model that has a small $r_\psi(\rvc | \rvz^{(1)})$. %

\paragraph{Conditions from manipulation constraints} Given a few labeled examples, here we wish to learn how to manipulate images. Specifically, we condition over the event that ``$\rvx$ has label $\rvc$ but is similar to image $\bar{\rvx}$''. Here $r_\psi(\rvc | \rvz^{(1)})$ is the unnormalized product between the classifier conditional probability and closeness to the latent $\bar{\rvz}^{(1)}$ of $\bar{\rvx}$ (\textit{e.g.}, measured with RBF kernel). We implement line 4 of Alg.~\ref{alg:few-shot-examples} with a Lanvegin-like procedure where we take a gradient step with respect to the classifier probability and then correct this gradient step with the diffusion model. Unlike many GAN-based methods~\cite{chen2017photographic,portenier2018faceshop,wang2018pix2pixHD,isola2017image,xia2021tedigan}, D2C does not need to optimize an inversion procedure at evaluation time, and thus the latent value is much faster to compute; D2C is also better at retaining fine-grained features of the original image due to the reconstruction loss. %

\section{Related Work}

\paragraph{Latent variable generative models} Most deep generative models explicitly define a latent representation, except for some energy-based models~\cite{hinton2002training,du2019implicit} and autoregressive models~\cite{oord2016pixel,oord2016wavenet,brown2020language}. 
Unlike VAEs and NFs, GANs do not explicitly define an inference model and instead optimize a two-player game. In terms of sample quality, GANs currently achieve superior performance over VAEs and NFs, but they can be difficult to invert even with additional optimization~\cite{karras2017progressive,xu2020generative,bau2019seeing}. This can be partially addressed by training reconstruction-based losses with GANs~\cite{larsen2016autoencoding,li2017alice}. Moreover, the GAN training procedure can be unstable~\cite{brock2016neural,brock2018large,miyato2018spectral}, lack a informative objective for measuring progress~\cite{arjovsky2017wasserstein}, and struggle with discrete data~\cite{yu2017seqgan}. Diffusion models~\cite{dhariwal2021diffusion} achieves high sample quality without adversarial training, but its latent dimension must be equal to the image dimension.

\paragraph{Addressing posterior mismatch in VAEs}
Most methods address this mismatch problem by improving inference models~\cite{mohamed2016learning,kingma2016improved,tomczak2016improving}, prior models~\cite{tomczak2017vae,aneja2020ncp,takahashi2019variational}, or objective functions~\cite{zhao2017infovae,zhao2017towards,zhao2018a,alemi2017fixing,makhzani2015adversarial}; all these approaches optimize the posterior model to be close to the prior. 
In Section~\ref{sec:prior-hole}, we explain why these approaches do not necessarily remove large ``prior holes'', so their sample qualities remain relatively poor even after many layers~\cite{vahdat2020nvae,child2020very}.
Other methods adopt a ``two-stage'' approach~\cite{dai2019diagnosing}, which fits a generative model over the latent space of autoencoders~\cite{oord2017neural,razavi2019generating,dhariwal2020jukebox,ramesh2021zero}.

\paragraph{Conditional generation with unconditional models}
To perform conditional generation over an unconditional LVGM, most methods assume access to a discriminative model (\textit{e.g.}, a classifier); the latent space of the generator is then modified to change the outputs of the discriminative model. The disciminative model can operate on either the image space~\cite{nguyen2017plug,patashnik2021styleclip,dhariwal2021diffusion} or the latent space~\cite{shen2020interpreting,xia2021tedigan}.  
For image space discriminative models, 
plug-and-play generative networks~\cite{nguyen2017plug} control the attributes of generated images via Langevin dynamics~\cite{roberts1998optimal}; these ideas are also explored in diffusion models~\cite{song2020score}. 
Image manipulation methods are based on GANs
 often operate with latent space discriminators~\cite{shen2020interpreting,xia2021tedigan}. However, these methods have some trouble manipulating real images because of imperfect reconstruction~\cite{zhu2019lia,bau2019seeing}. This is not a problem in D2C since a reconstruction objective is optimized.

\section{Experiments}

We examine the conditional and unconditional generation qualities of D2C over CIFAR-10~\cite{krizhevsky2012imagenet}, CIFAR-100~\cite{krizhevsky2012imagenet}, fMoW~\cite{christie2018functional}, CelebA-64~\cite{liu2015deep}, CelebA-HQ-256~\cite{karras2017progressive}, and FFHQ-256~\cite{karras2018a}. 
Our D2C implementation is based on the state-of-the-art NVAE~\cite{vahdat2020nvae} autoencoder structure, the U-Net diffusion model~\cite{ho2020denoising}, and the MoCo-v2 contrastive representation learning method~\cite{chen2020improved}. 
We keep the diffusion series hyperparameter $\{\alpha_i\}_{i=1}^{T}$ identical to ensure a fair comparison with different diffusion models. 
For the contrastive weight hyperparameter $\lambda$ in \Eqref{eq:d2c-obj}, we consider the value of $\lambda = 10^{-4}$ based on the relative scale between the $L_{\mathrm{C}}$ and $L_{\mathrm{D2}}$; we find that the results are relatively insensitive to $\lambda$. We use 100 diffusion steps for DDIM and D2C unless mentioned otherwise, as running with longer steps is not computationally economical despite tiny gains in FID~\cite{song2020denoising}.
We include additional training details, such as architectures, optimizers and learning rates in Appendix~\ref{app:exp}.

\begin{table}[htbp]
\centering
\caption{Quality of representations and generations with LVGMs.}%
\label{tab:nvae-ddim-d2c-3-datasets}
\vspace{.2em}
\begin{tabular}{l|ccc|ccc|ccc}
\toprule
\multirow{2}*{Model} & \multicolumn{3}{c}{CIFAR-10} & \multicolumn{3}{|c}{CIFAR-100} & \multicolumn{3}{|c}{fMoW}\\
  & FID $\downarrow$ & MSE $\downarrow$ & Acc $\uparrow$  & FID $\downarrow$ & MSE $\downarrow$ & Acc $\uparrow$ & FID $\downarrow$ & MSE $\downarrow$ & Acc $\uparrow$\\\midrule
   NVAE~\cite{vahdat2020nvae} & 36.4 & \textbf{0.25} & 18.8 & 42.5 & 0.53 & 4.1 & 82.25 & \textbf{0.30} & 27.7\\
   DDIM~\cite{song2020denoising} & \textbf{4.16} & 2.5  & 22.5 & \textbf{10.16} & 3.2 & 2.2 & \textbf{37.74} & 3.0 & 23.5\\
    \midrule
     D2C (Ours) & 10.15 & 0.76 & \textbf{76.02}  & 14.62  & \textbf{0.44} & \textbf{42.75}  & 44.7  & 2.33 & \textbf{66.9}\\
    \bottomrule
\end{tabular}
\end{table}

\subsection{Unconditional generation}
For unconditional generation, we measure the sample quality of images using the Frechet Inception Distance (FID,~\cite{heusel2017gans}) with 50,000 images. In particular, we extensively evaluate NVAE~\cite{vahdat2020nvae} and DDIM~\cite{song2020denoising}, a competitive VAE model and a competitive diffusion model %
as baselines because we can directly obtain features from them without additional optimization steps\footnote{For DDIM, the latent representations $\rvx^{(0)}$ are obtained by reversing the neural ODE process.}.
For them, we report mean-squared reconstruction error (MSE, summed over all pixels, pixels normalized to $[0, 1]$) and linear classification accuracy (Acc., measured in percentage) over $\rvz_1$ features for the test set. 
\begin{figure}
    \centering
    \includegraphics[width=\textwidth]{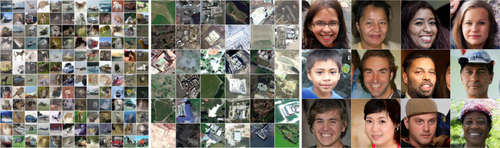}
    \caption{Generated samples on CIFAR-10 (left), fMoW (mid), and FFHQ $256 \times 256$ (right).}
    \label{fig:generated_small}
\end{figure}

We report sample quality results\footnote{Due to space limits, we place additional CIFAR-10 results in Appendix~\ref{app:results}.} in Tables~\ref{tab:nvae-ddim-d2c-3-datasets}, and~\ref{tab:large-datasets-fid}. For FID, we outperform NVAE in all datasets and outperform DDIM on CelebA-64 and CelebA-HQ-256, which suggests our results are competitive with state-of-the-art non-adversarial generative models.
In Table~\ref{tab:nvae-ddim-d2c-3-datasets}, we additionally compare NVAE, DDIM and D2C in terms of reconstruction and linear classification accuracy. As all three methods contain reconstruction losses, the MSE values are low and comparable. However, D2C enjoys much better linear classification accuracy than the other two thanks to the contrastive SSL component. We further note that training the same contrastive SSL method without $L_{\mathrm{D2}}$ achieves slightly higher $78.3\%$ accuracy on CIFAR-10. We tried improving this via ResNet~\cite{he2015deep} encoders, but this significantly increased reconstruction error, possibly due to loss of information in average pooling layers.

\begin{table}[htbp]
\centering
\caption{FID scores over different faces dataset with LVGMs.} %
\label{tab:large-datasets-fid}
\vspace{.2em}
\begin{tabular}{l|ccc}
\toprule
Model & CelebA-64 & CelebA-HQ-256 & FFHQ-256\\\midrule
   NVAE~\cite{vahdat2020nvae} & 13.48 & 40.26 & 26.02 \\
   DDIM~\cite{song2020denoising} & 6.53 & 25.6 & - \\
    \midrule
     D2C (Ours) & \textbf{5.7} & \textbf{18.74} & \textbf{13.04}\\
    \bottomrule
\end{tabular}
\end{table}

\begin{table}[H]
    \centering
    \caption{Sample quality as a function of diffusion steps.}
    \vspace{.2em}
    \label{tab:diffusion-steps-ablation}
    \begin{tabular}{l|ccc|ccc|ccc}
    \toprule
      & \multicolumn{3}{c|}{CIFAR-10} & \multicolumn{3}{c|}{CIFAR-100} & \multicolumn{3}{c}{CelebA-64} \\
     Steps & 10  & 50 & 100 & 10 & 50 & 100 & 10 & 50 & 100\\\midrule
     DDPM~\cite{ho2020denoising} & 41.07 & 8.01 & 5.78 & 50.27 & 21.37 & 16.72 & 33.12 & 18.48 & 13.93\\
     DDIM~\cite{song2020denoising} & \textbf{13.36} & \textbf{4.67} & \textbf{4.16} & \textbf{23.34} & \textbf{11.69} & \textbf{10.16} & 17.33 & 9.17 & 6.53 \\
     D2C (Ours) & 17.71 & 10.11 & 10.15 & 23.16 & 14.62 & 14.46 &  \textbf{17.32} & \textbf{6.8} & \textbf{5.7} \\
        \bottomrule
    \end{tabular}
\end{table}

\subsection{Few-shot conditional generation from examples}

We demonstrate the advantage of D2C representations by performing few-shot conditional generation over labels. We consider two types of labeled examples: one has binary labels for which we train a binary classifier; the other is positive-only labeled (\textit{e.g.}, images of female class) for which we train a PU classifier. Our goal here is to generate a diverse group of images with a certain label. We evaluate and compare three models: D2C, NVAE and DDIM. We train a classifier $r_\psi(\rvc | \rvz)$ over the latent space of these models; we also train a image space classifier and use it with DDIM (denoted as DDIM-I). We run Algorithm~\ref{alg:few-shot-examples} for these models, where line 4 is implemented via rejection sampling. As our goal is to compare different models, we leave more sophisticated methods~\cite{dhariwal2021diffusion} as future work.

We consider performing 8 conditional generation tasks over CelebA-64 with 2 binary classifiers (trained over 100 samples, 50 for each class) and 4 PU classifiers (trained over 100 positively labeled and 10k unlabeled samples).
We also report a ``naive'' approach where we use all the training images (regardless of labels) and compute its FID with the corresponding subset of images (\textit{e.g.}, all images versus blond images).
In Table~\ref{tab:few-shot-label}, we report the FID score between generated images (5k samples) and real images of the corresponding label. These results suggest that D2C outperforms the other approaches, and is the only one that performs better than the ``naive'' approach in most cases, illustrating the advantage of contrastive representations for few-shot conditional generation.
\begin{table}[h]
\centering
\caption{FID scores for few-shot conditional generation with various types of labeled examples. Naive performs very well for non-blond due to class percentages.
}
\label{tab:few-shot-label}
    \vspace{0.2em}
\begin{tabular}{l|l|cccc|c}
\toprule
{Method} & {Classes (\% in train set)}       &{D2C} & {DDIM} & {NVAE} & {DDIM-I} & Naive \\\midrule
\multirow{4}{*}{Binary} & \multicolumn{1}{l|}{Male ($42\%$)}      & \textbf{13.44}                & 38.38                 &  41.07               & 29.03     & 26.34                          \\
                                       & \multicolumn{1}{l|}{Female ($58\%$)}    & \textbf{9.51}                 & 19.25                 & 16.57                      & 15.17  & 18.72                                   \\ \cmidrule(l){2-7} 
                                       & \multicolumn{1}{l|}{Blond ($15\%$)}     & \textbf{17.61}                & 31.39                 &  31.24                     & 29.09 & 27.51                                  \\
                                       & \multicolumn{1}{l|}{Non-Blond ($85\%$)} & \textbf{8.94}                 & 9.67                 & 16.73                       & 19.76 & 3.77                                 \\ \midrule
\multirow{4}{*}{PU}     & \multicolumn{1}{l|}{Male ($42\%$)}      & \textbf{16.39}                & 37.03                 &      42.78                 & 19.60 & 26.34                                     \\
                                       & \multicolumn{1}{l|}{Female ($58\%$)}    & \textbf{12.21}                & 15.42                  &  18.36                     & 14.96 & 18.72                                      \\ \cmidrule(l){2-7} 
                                       & \multicolumn{1}{l|}{Blond ($15\%$)}     & \textbf{10.09}                & 30.20                 & 31.06                      & 76.52   & 27.51                               \\
                                       & \multicolumn{1}{l|}{Non-Blond ($85\%$)} & \textbf{9.09}                 & 9.70                &   17.98                    & 9.90 & 3.77\\                        \bottomrule         
\end{tabular}
\end{table}
\subsection{Few-shot conditional generation from manipulation constraints}

\begin{figure}
    \centering
    \includegraphics[width=\textwidth]{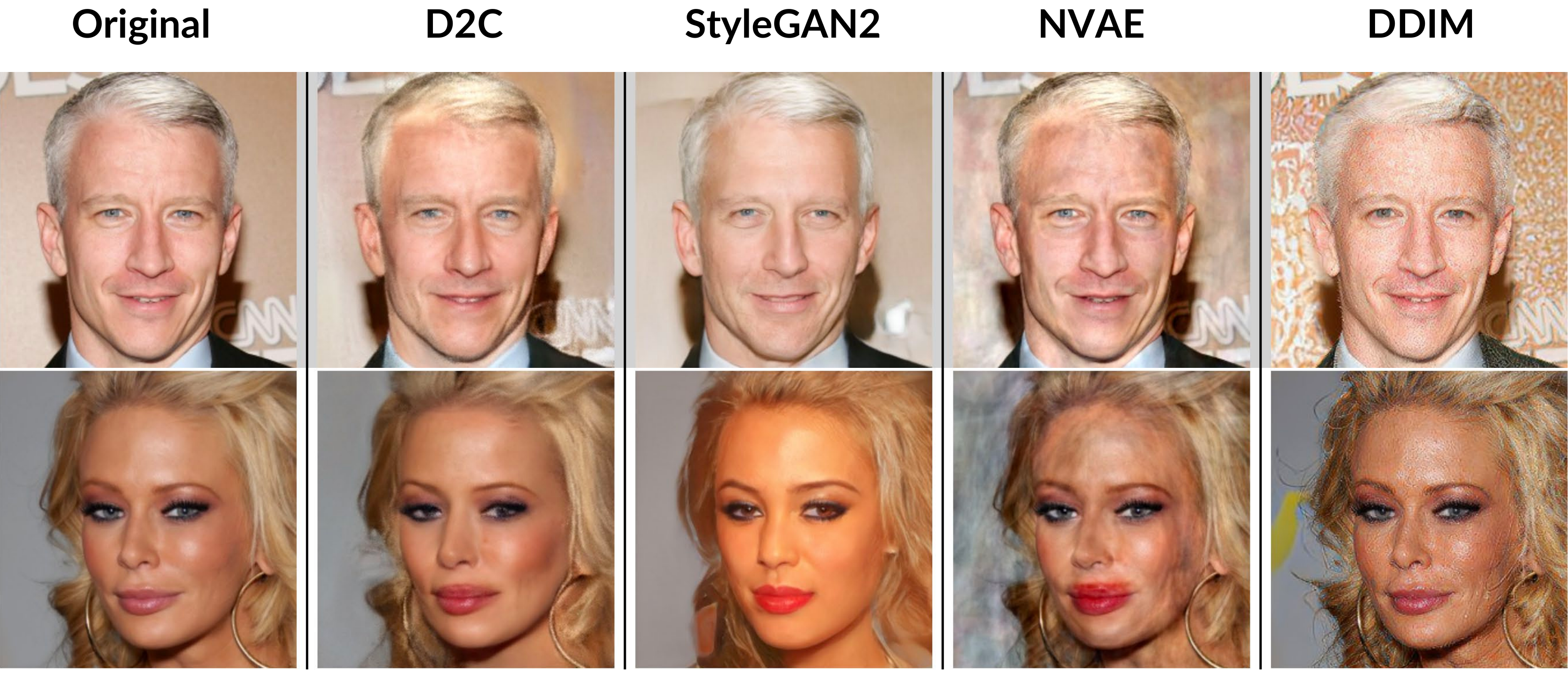}
    \caption{Image manipulation results for \textit{blond} (top) and \textit{red lipstick} (bottom). D2C is better than StyleGAN2 at preserving details of the original image, such as eyes, earrings, and background.}
    \label{fig:image-manip-demo}
\end{figure}

\begin{figure}
    \centering
    \includegraphics[width=0.9\textwidth]{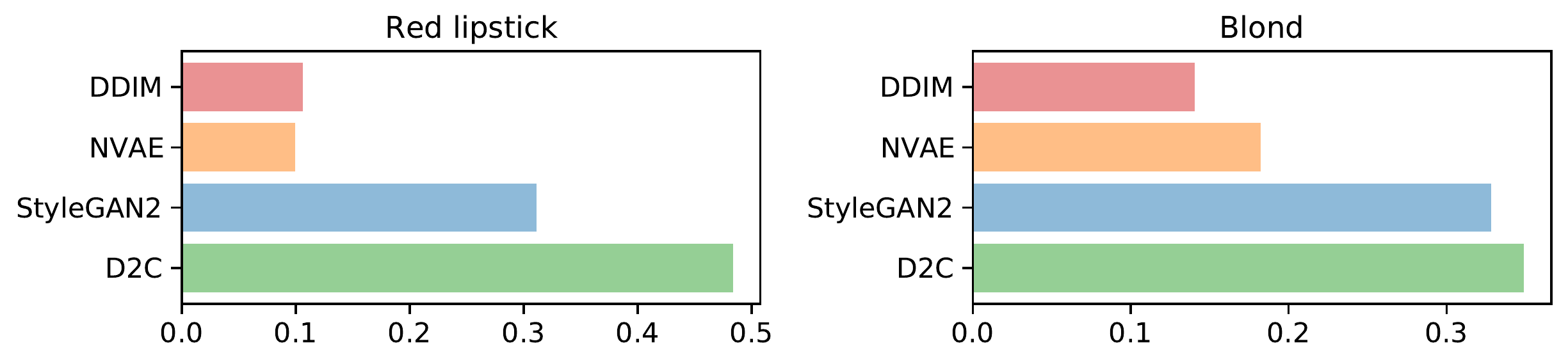}
    \caption{AMT evaluation over image manipulations. $x$-axis shows the percentage that the evaluator selects the image generated from the corresponding model out of 4 images from each model.}
    \label{fig:amt}
\end{figure}

Finally, we consider image manipulation where we use binary classifiers that are learned over 50 labeled instances for each class. We perform Amazon Mechanical Turk (AMT) evaluations over two attributes in the CelebA-256 dataset, \textit{blond} and \textit{red lipstick}, over D2C, DDIM, NVAE and StyleGAN2~\cite{karras2020analyzing} (see Figure~\ref{fig:image-manip-demo}). 
The evaluation is double-blinded: neither we nor the evaluators know the correspondence between generated image and underlying model during the study.
We include more details (algorithm, setup and human evaluation) in Appendix~\ref{app:exp} and additional qualitative results (such as \textit{beard} and \textit{gender} attributes) in Appendix~\ref{app:results}. 

In Figure~\ref{fig:amt}, we show the percentage of manipulations preferred by AMT evaluators for each model; D2C slightly outperforms StyleGAN2 for \textit{blond} and significantly outperforms StyleGAN2 for \textit{red lipstick}. When we compare D2C with only StyleGAN2, D2C is preferred over $51.5\%$ for \textit{blond} and $60.8\%$ for \textit{red lipstick}.
An additional advantage of D2C is that the manipulation is much faster than StyleGAN2, since the latter requires additional optimization over the latent space to improve reconstruction~\cite{zhu2020domain}. On the same Nvidia 1080Ti GPU, it takes 0.013 seconds to obtain the latent code in D2C, while the same takes 8 seconds~\cite{zhu2020domain} for StyleGAN2 ($615 \times$ slower). As decoding is very fast for both models, D2C generations are around two orders of magnitude faster to produce.

\section{Discussions and Limitations}\label{sec:discussion}
We introduced D2C, a VAE-based generative model with a latent space suitable for few-shot conditional generation. 
To our best knowledge, our model is the first unconditional VAE to demonstrate superior image manipulation performance than StyleGAN2, which is surprising given our use of a regular NVAE architecture. We believe that with better architectures, such as designs from StyleGAN2 or Transformers~\cite{hudson2021generative}, D2C can achieve even better performance.
It is also interesting to formally investigate the integration between D2C and other types of conditions on the latent space, as well as training D2C in conjunction with other domains and data modalities, such as text~\cite{ramesh2021zero}, in a fashion that is similar to semi-supervised learning. Nevertheless, we note that our model have to be used properly in order to mitigate potential negative societal impacts, such as deep fakes.

\section*{Acknowledgements}
This research was supported by NSF (\#1651565, \#1522054, \#1733686), ONR (N00014-19-1-2145), AFOSR (FA9550-19-1-0024), Amazon AWS, Stanford Institute for Human-Centered Artificial Intelligence (HAI), and Google Cloud.

\bibliographystyle{plain}
\bibliography{bib}

\begin{thebibliography}{100}

\bibitem{alemi2017fixing}
Alexander~A Alemi, Ben Poole, Ian Fischer, Joshua~V Dillon, Rif~A Saurous, and
  Kevin Murphy.
\newblock Fixing a broken {ELBO}.
\newblock {\em arXiv preprint arXiv:1711.00464}, November 2017.

\bibitem{aneja2019sequential}
Jyoti Aneja, Harsh Agrawal, Dhruv Batra, and Alexander Schwing.
\newblock Sequential latent spaces for modeling the intention during diverse
  image captioning.
\newblock In {\em Proceedings of the IEEE/CVF International Conference on
  Computer Vision}, pages 4261--4270, 2019.

\bibitem{aneja2020ncp}
Jyoti Aneja, Alexander Schwing, Jan Kautz, and Arash Vahdat.
\newblock {NCP-VAE}: Variational autoencoders with noise contrastive priors.
\newblock {\em arXiv preprint arXiv:2010.02917}, October 2020.

\bibitem{arjovsky2017wasserstein}
Martin Arjovsky, Soumith Chintala, and L{\'e}on Bottou.
\newblock Wasserstein {GAN}.
\newblock {\em arXiv preprint arXiv:1701.07875}, January 2017.

\bibitem{asim2019invertible}
Muhammad Asim, Ali Ahmed, and Paul Hand.
\newblock Invertible generative models for inverse problems: mitigating
  representation error and dataset bias.
\newblock {\em arXiv preprint arXiv:1905.11672}, May 2019.

\bibitem{bartunov2018few}
Sergey Bartunov and Dmitry Vetrov.
\newblock Few-shot generative modelling with generative matching networks.
\newblock In Amos Storkey and Fernando Perez-Cruz, editors, {\em Proceedings of
  the {Twenty-First} International Conference on Artificial Intelligence and
  Statistics}, volume~84 of {\em Proceedings of Machine Learning Research},
  pages 670--678. PMLR, 2018.

\bibitem{bau2019seeing}
David Bau, Jun-Yan Zhu, Jonas Wulff, William Peebles, Hendrik Strobelt, Bolei
  Zhou, and Antonio Torralba.
\newblock Seeing what a gan cannot generate.
\newblock In {\em Proceedings of the IEEE/CVF International Conference on
  Computer Vision}, pages 4502--4511, 2019.

\bibitem{brock2018large}
Andrew Brock, Jeff Donahue, and Karen Simonyan.
\newblock Large scale {GAN} training for high fidelity natural image synthesis.
\newblock {\em arXiv preprint arXiv:1809.11096}, September 2018.

\bibitem{brock2016neural}
Andrew Brock, Theodore Lim, James~M Ritchie, and Nick Weston.
\newblock Neural photo editing with introspective adversarial networks.
\newblock {\em arXiv preprint arXiv:1609.07093}, 2016.

\bibitem{brown2020language}
Tom~B Brown, Benjamin Mann, Nick Ryder, Melanie Subbiah, Jared Kaplan, Prafulla
  Dhariwal, Arvind Neelakantan, Pranav Shyam, Girish Sastry, Amanda Askell,
  Sandhini Agarwal, Ariel Herbert-Voss, Gretchen Krueger, Tom Henighan, Rewon
  Child, Aditya Ramesh, Daniel~M Ziegler, Jeffrey Wu, Clemens Winter,
  Christopher Hesse, Mark Chen, Eric Sigler, Mateusz Litwin, Scott Gray,
  Benjamin Chess, Jack Clark, Christopher Berner, Sam McCandlish, Alec Radford,
  Ilya Sutskever, and Dario Amodei.
\newblock Language models are {Few-Shot} learners.
\newblock {\em arXiv preprint arXiv:2005.14165}, May 2020.

\bibitem{candes2006stable}
Emmanuel~J Candes, Justin~K Romberg, and Terence Tao.
\newblock Stable signal recovery from incomplete and inaccurate measurements.
\newblock {\em Communications on Pure and Applied Mathematics: A Journal Issued
  by the Courant Institute of Mathematical Sciences}, 59(8):1207--1223, 2006.

\bibitem{chen2017photographic}
Qifeng Chen and Vladlen Koltun.
\newblock Photographic image synthesis with cascaded refinement networks.
\newblock In {\em ICCV}, 2017.

\bibitem{chen2018neural}
Ricky T~Q Chen, Yulia Rubanova, Jesse Bettencourt, and David Duvenaud.
\newblock Neural ordinary differential equations.
\newblock {\em arXiv preprint arXiv:1806.07366}, June 2018.

\bibitem{chen2020a}
Ting Chen, Simon Kornblith, Mohammad Norouzi, and Geoffrey Hinton.
\newblock A simple framework for contrastive learning of visual
  representations.
\newblock {\em arXiv preprint arXiv:2002.05709}, February 2020.

\bibitem{chen2020improved}
Xinlei Chen, Haoqi Fan, Ross Girshick, and Kaiming He.
\newblock Improved baselines with momentum contrastive learning.
\newblock {\em arXiv preprint arXiv:2003.04297}, March 2020.

\bibitem{chen2020exploring}
Xinlei Chen and Kaiming He.
\newblock Exploring simple siamese representation learning.
\newblock {\em arXiv preprint arXiv:2011.10566}, November 2020.

\bibitem{chen2021an}
Xinlei Chen, Saining Xie, and Kaiming He.
\newblock An empirical study of training {Self-Supervised} vision transformers.
\newblock {\em arXiv preprint arXiv:2104.02057}, April 2021.

\bibitem{child2020very}
Rewon Child.
\newblock Very deep vaes generalize autoregressive models and can outperform
  them on images.
\newblock {\em arXiv preprint arXiv:2011.10650}, 2020.

\bibitem{christie2018functional}
Gordon Christie, Neil Fendley, James Wilson, and Ryan Mukherjee.
\newblock Functional map of the world.
\newblock In {\em Proceedings of the IEEE Conference on Computer Vision and
  Pattern Recognition}, pages 6172--6180, 2018.

\bibitem{clouatre2019figr}
Louis Clou{\^a}tre and Marc Demers.
\newblock Figr: Few-shot image generation with reptile.
\newblock {\em arXiv preprint arXiv:1901.02199}, 2019.

\bibitem{dai2019diagnosing}
Bin Dai and David Wipf.
\newblock Diagnosing and enhancing {VAE} models.
\newblock {\em arXiv preprint arXiv:1903.05789}, March 2019.

\bibitem{daras2021intermediate}
Giannis Daras, Joseph Dean, Ajil Jalal, and Alexandros~G Dimakis.
\newblock Intermediate layer optimization for inverse problems using deep
  generative models.
\newblock {\em arXiv preprint arXiv:2102.07364}, February 2021.

\bibitem{devlin2018bert}
Jacob Devlin, Ming-Wei Chang, Kenton Lee, and Kristina Toutanova.
\newblock {BERT}: Pre-training of deep bidirectional transformers for language
  understanding.
\newblock {\em arXiv preprint arXiv:1810.04805}, October 2018.

\bibitem{dhariwal2020jukebox}
Prafulla Dhariwal, Heewoo Jun, Christine Payne, Jong~Wook Kim, Alec Radford,
  and Ilya Sutskever.
\newblock Jukebox: A generative model for music.
\newblock {\em arXiv preprint arXiv:2005.00341}, 2020.

\bibitem{dhariwal2021diffusion}
Prafulla Dhariwal and Alex Nichol.
\newblock Diffusion models beat {GANs} on image synthesis.
\newblock {\em arXiv preprint arXiv:2105.05233}, May 2021.

\bibitem{dinh2014nice}
L~Dinh, D~Krueger, and Y~Bengio.
\newblock {NICE}: Non-linear independent components estimation.
\newblock {\em arXiv preprint arXiv:1410.8516}, 2014.

\bibitem{dinh2016density}
Laurent Dinh, Jascha Sohl-Dickstein, and Samy Bengio.
\newblock Density estimation using real {NVP}.
\newblock {\em arXiv preprint arXiv:1605.08803}, May 2016.

\bibitem{donahue2016adversarial}
Jeff Donahue, Philipp Kr{\"a}henb{\"u}hl, and Trevor Darrell.
\newblock Adversarial feature learning.
\newblock {\em arXiv preprint arXiv:1605.09782}, May 2016.

\bibitem{du2019implicit}
Yilun Du and Igor Mordatch.
\newblock Implicit generation and generalization in {Energy-Based} models.
\newblock {\em arXiv preprint arXiv:1903.08689}, March 2019.

\bibitem{dumoulin2016adversarially}
Vincent Dumoulin, Ishmael Belghazi, Ben Poole, Olivier Mastropietro, Alex Lamb,
  Martin Arjovsky, and Aaron Courville.
\newblock Adversarially learned inference.
\newblock {\em arXiv preprint arXiv:1606.00704}, June 2016.

\bibitem{elkan2008learning}
Charles Elkan and Keith Noto.
\newblock Learning classifiers from only positive and unlabeled data.
\newblock In {\em 14th ACM SIGKDD}, pages 213--220, 2008.

\bibitem{goodfellow2014generative}
Ian Goodfellow, Jean Pouget-Abadie, Mehdi Mirza, Bing Xu, David Warde-Farley,
  Sherjil Ozair, Aaron Courville, and Yoshua Bengio.
\newblock Generative adversarial nets.
\newblock In Z~Ghahramani, M~Welling, C~Cortes, N~D Lawrence, and K~Q
  Weinberger, editors, {\em Advances in Neural Information Processing Systems
  27}, pages 2672--2680. Curran Associates, Inc., 2014.

\bibitem{grill2020bootstrap}
Jean-Bastien Grill, Florian Strub, Florent Altch{\'e}, Corentin Tallec,
  Pierre~H Richemond, Elena Buchatskaya, Carl Doersch, Bernardo~Avila Pires,
  Zhaohan~Daniel Guo, Mohammad~Gheshlaghi Azar, Bilal Piot, Koray Kavukcuoglu,
  R{\'e}mi Munos, and Michal Valko.
\newblock Bootstrap your own latent: A new approach to {Self-Supervised}
  learning.
\newblock {\em arXiv preprint arXiv:2006.07733}, June 2020.

\bibitem{grover2019graphite}
Aditya Grover, Aaron Zweig, and Stefano Ermon.
\newblock Graphite: Iterative generative modeling of graphs.
\newblock In {\em International conference on machine learning}, pages
  2434--2444. PMLR, 2019.

\bibitem{he2019momentum}
Kaiming He, Haoqi Fan, Yuxin Wu, Saining Xie, and Ross Girshick.
\newblock Momentum contrast for unsupervised visual representation learning.
\newblock {\em arXiv preprint arXiv:1911.05722}, November 2019.

\bibitem{he2015deep}
Kaiming He, Xiangyu Zhang, Shaoqing Ren, and Jian Sun.
\newblock Deep residual learning for image recognition.
\newblock {\em arXiv preprint arXiv:1512.03385}, December 2015.

\bibitem{henaff2020data}
Olivier Henaff.
\newblock Data-efficient image recognition with contrastive predictive coding.
\newblock In {\em International Conference on Machine Learning}, pages
  4182--4192. PMLR, 2020.

\bibitem{heusel2017gans}
Martin Heusel, Hubert Ramsauer, Thomas Unterthiner, Bernhard Nessler, and Sepp
  Hochreiter.
\newblock {GANs} trained by a two {Time-Scale} update rule converge to a local
  nash equilibrium.
\newblock {\em arXiv preprint arXiv:1706.08500}, June 2017.

\bibitem{hinton2002training}
Geoffrey~E Hinton.
\newblock Training products of experts by minimizing contrastive divergence.
\newblock {\em Neural computation}, 14(8):1771--1800, August 2002.

\bibitem{ho2020denoising}
Jonathan Ho, Ajay Jain, and Pieter Abbeel.
\newblock Denoising diffusion probabilistic models.
\newblock {\em arXiv preprint arXiv:2006.11239}, June 2020.

\bibitem{hoffman2016elbo}
Matthew~D Hoffman and Matthew~J Johnson.
\newblock Elbo surgery: yet another way to carve up the variational evidence
  lower bound.
\newblock In {\em Workshop in Advances in Approximate Bayesian Inference,
  {NIPS}}, volume~1, page~2. approximateinference.org, 2016.

\bibitem{hudson2021generative}
Drew~A Hudson and C~Lawrence Zitnick.
\newblock Generative adversarial transformers.
\newblock {\em arXiv preprint arXiv:2103.01209}, 2021.

\bibitem{isola2017image}
Phillip Isola, Jun-Yan Zhu, Tinghui Zhou, and Alexei~A Efros.
\newblock Image-to-image translation with conditional adversarial networks.
\newblock In {\em CVPR}, 2017.

\bibitem{kadkhodaie2020solving}
Zahra Kadkhodaie and Eero~P Simoncelli.
\newblock Solving linear inverse problems using the prior implicit in a
  denoiser.
\newblock {\em arXiv preprint arXiv:2007.13640}, July 2020.

\bibitem{karras2017progressive}
Tero Karras, Timo Aila, Samuli Laine, and Jaakko Lehtinen.
\newblock Progressive growing of {GANs} for improved quality, stability, and
  variation.
\newblock {\em arXiv preprint arXiv:1710.10196}, October 2017.

\bibitem{karras2018a}
Tero Karras, Samuli Laine, and Timo Aila.
\newblock A {Style-Based} generator architecture for generative adversarial
  networks.
\newblock {\em arXiv preprint arXiv:1812.04948}, December 2018.

\bibitem{karras2020analyzing}
Tero Karras, Samuli Laine, Miika Aittala, Janne Hellsten, Jaakko Lehtinen, and
  Timo Aila.
\newblock Analyzing and improving the image quality of stylegan.
\newblock In {\em Proceedings of the IEEE/CVF Conference on Computer Vision and
  Pattern Recognition}, pages 8110--8119, 2020.

\bibitem{kingma2016improved}
Diederik~P Kingma, Tim Salimans, Rafal Jozefowicz, Xi~Chen, Ilya Sutskever, and
  Max Welling.
\newblock Improved variational inference with inverse autoregressive flow.
\newblock In D~D Lee, M~Sugiyama, U~V Luxburg, I~Guyon, and R~Garnett, editors,
  {\em Advances in Neural Information Processing Systems 29}, pages 4743--4751.
  Curran Associates, Inc., 2016.

\bibitem{kingma2013auto}
Diederik~P Kingma and Max Welling.
\newblock {Auto-Encoding} variational bayes.
\newblock {\em arXiv preprint arXiv:1312.6114v10}, December 2013.

\bibitem{krizhevsky2012imagenet}
Alex Krizhevsky, Ilya Sutskever, and Geoffrey~E Hinton.
\newblock {ImageNet} classification with deep convolutional neural networks.
\newblock In F~Pereira, C~J~C Burges, L~Bottou, and K~Q Weinberger, editors,
  {\em Advances in Neural Information Processing Systems 25}, pages 1097--1105.
  Curran Associates, Inc., 2012.

\bibitem{larsen2016autoencoding}
Anders Boesen~Lindbo Larsen, S{\o}ren~Kaae S{\o}nderby, Hugo Larochelle, and
  Ole Winther.
\newblock Autoencoding beyond pixels using a learned similarity metric.
\newblock In {\em International conference on machine learning}, pages
  1558--1566. PMLR, 2016.

\bibitem{li2017alice}
C~Li, H~Liu, C~Chen, Y~Pu, L~Chen, and {others}.
\newblock Alice: Towards understanding adversarial learning for joint
  distribution matching.
\newblock {\em Advances in neural information processing systems}, 2017.

\bibitem{li2020optimus}
Chunyuan Li, Xiang Gao, Yuan Li, Baolin Peng, Xiujun Li, Yizhe Zhang, and
  Jianfeng Gao.
\newblock Optimus: Organizing sentences via pre-trained modeling of a latent
  space.
\newblock {\em arXiv preprint arXiv:2004.04092}, 2020.

\bibitem{lin2014microsoft}
Tsung-Yi Lin, Michael Maire, Serge Belongie, James Hays, Pietro Perona, Deva
  Ramanan, Piotr Doll{\'a}r, and C~Lawrence Zitnick.
\newblock Microsoft coco: Common objects in context.
\newblock In {\em European conference on computer vision}, pages 740--755.
  Springer, 2014.

\bibitem{liu2015deep}
Ziwei Liu, Ping Luo, Xiaogang Wang, and Xiaoou Tang.
\newblock Deep learning face attributes in the wild.
\newblock In {\em Proceedings of the IEEE international conference on computer
  vision}, pages 3730--3738, 2015.

\bibitem{makhzani2015adversarial}
Alireza Makhzani, Jonathon Shlens, Navdeep Jaitly, Ian Goodfellow, and Brendan
  Frey.
\newblock Adversarial autoencoders.
\newblock {\em arXiv preprint arXiv:1511.05644}, 2015.

\bibitem{mansimov2015generating}
Elman Mansimov, Emilio Parisotto, Jimmy~Lei Ba, and Ruslan Salakhutdinov.
\newblock Generating images from captions with attention.
\newblock {\em arXiv preprint arXiv:1511.02793}, 2015.

\bibitem{mirza2014conditional}
Mehdi Mirza and Simon Osindero.
\newblock Conditional generative adversarial nets.
\newblock {\em arXiv preprint arXiv:1411.1784}, November 2014.

\bibitem{mittal2021symbolic}
Gautam Mittal, Jesse Engel, Curtis Hawthorne, and Ian Simon.
\newblock Symbolic music generation with diffusion models.
\newblock {\em arXiv preprint arXiv:2103.16091}, March 2021.

\bibitem{miyato2018spectral}
Takeru Miyato, Toshiki Kataoka, Masanori Koyama, and Yuichi Yoshida.
\newblock Spectral normalization for generative adversarial networks.
\newblock {\em arXiv preprint arXiv:1802.05957}, February 2018.

\bibitem{mohamed2016learning}
Shakir Mohamed and Balaji Lakshminarayanan.
\newblock Learning in implicit generative models.
\newblock {\em arXiv preprint arXiv:1610.03483}, October 2016.

\bibitem{najibi2020racial}
Alex Najibi.
\newblock {Racial Discrimination in Face Recognition Technology}, 2020.

\bibitem{nguyen2017plug}
Anh Nguyen, Jeff Clune, Yoshua Bengio, Alexey Dosovitskiy, and Jason Yosinski.
\newblock Plug \& play generative networks: Conditional iterative generation of
  images in latent space.
\newblock In {\em Proceedings of the IEEE Conference on Computer Vision and
  Pattern Recognition}, pages 4467--4477, 2017.

\bibitem{niu2020permutation}
Chenhao Niu, Yang Song, Jiaming Song, Shengjia Zhao, Aditya Grover, and Stefano
  Ermon.
\newblock Permutation invariant graph generation via score-based generative
  modeling.
\newblock In {\em International Conference on Artificial Intelligence and
  Statistics}, pages 4474--4484. PMLR, 2020.

\bibitem{noroozi2016unsupervised}
Mehdi Noroozi and Paolo Favaro.
\newblock Unsupervised learning of visual representations by solving jigsaw
  puzzles.
\newblock In {\em European conference on computer vision}, pages 69--84.
  Springer, 2016.

\bibitem{park2020swapping}
Taesung Park, Jun-Yan Zhu, Oliver Wang, Jingwan Lu, Eli Shechtman, Alexei~A
  Efros, and Richard Zhang.
\newblock Swapping autoencoder for deep image manipulation.
\newblock {\em arXiv preprint arXiv:2007.00653}, July 2020.

\bibitem{patashnik2021styleclip}
Or~Patashnik, Zongze Wu, Eli Shechtman, Daniel Cohen-Or, and Dani Lischinski.
\newblock {StyleCLIP}: {Text-Driven} manipulation of {StyleGAN} imagery.
\newblock {\em arXiv preprint arXiv:2103.17249}, March 2021.

\bibitem{ping2020waveflow}
Wei Ping, Kainan Peng, Kexin Zhao, and Zhao Song.
\newblock Waveflow: A compact flow-based model for raw audio.
\newblock In {\em International Conference on Machine Learning}, pages
  7706--7716. PMLR, 2020.

\bibitem{portenier2018faceshop}
Tiziano Portenier, Qiyang Hu, Attila Szab\'{o}, Siavash~Arjomand Bigdeli, Paolo
  Favaro, and Matthias Zwicker.
\newblock Faceshop: Deep sketch-based face image editing.
\newblock {\em ACM Transactions on Graphics}, 37(4), 2018.

\bibitem{radford2021learning}
Alec Radford, Jong~Wook Kim, Chris Hallacy, Aditya Ramesh, Gabriel Goh,
  Sandhini Agarwal, Girish Sastry, Amanda Askell, Pamela Mishkin, Jack Clark,
  Gretchen Krueger, and Ilya Sutskever.
\newblock Learning transferable visual models from natural language
  supervision.
\newblock {\em arXiv preprint arXiv:2103.00020}, February 2021.

\bibitem{ramesh2021zero}
Aditya Ramesh, Mikhail Pavlov, Gabriel Goh, Scott Gray, Chelsea Voss, Alec
  Radford, Mark Chen, and Ilya Sutskever.
\newblock {Zero-Shot} {Text-to-Image} generation.
\newblock {\em arXiv preprint arXiv:2102.12092}, February 2021.

\bibitem{razavi2019generating}
Ali Razavi, Aaron van~den Oord, and Oriol Vinyals.
\newblock Generating diverse high-fidelity images with vq-vae-2.
\newblock {\em arXiv preprint arXiv:1906.00446}, 2019.

\bibitem{reed2016generative}
Scott Reed, Zeynep Akata, Xinchen Yan, Lajanugen Logeswaran, Bernt Schiele, and
  Honglak Lee.
\newblock Generative adversarial text to image synthesis.
\newblock In {\em International Conference on Machine Learning}, pages
  1060--1069. PMLR, 2016.

\bibitem{rezende2015variational}
Danilo~Jimenez Rezende and Shakir Mohamed.
\newblock Variational inference with normalizing flows.
\newblock {\em arXiv preprint arXiv:1505.05770}, May 2015.

\bibitem{roberts1998optimal}
Gareth~O Roberts and Jeffrey~S Rosenthal.
\newblock Optimal scaling of discrete approximations to langevin diffusions.
\newblock {\em Journal of the Royal Statistical Society: Series B (Statistical
  Methodology)}, 60(1):255--268, 1998.

\bibitem{rosca2018distribution}
Mihaela Rosca, Balaji Lakshminarayanan, and Shakir Mohamed.
\newblock Distribution matching in variational inference.
\newblock {\em arXiv preprint arXiv:1802.06847}, February 2018.

\bibitem{shen2020interpreting}
Yujun Shen, Jinjin Gu, Xiaoou Tang, and Bolei Zhou.
\newblock Interpreting the latent space of gans for semantic face editing.
\newblock In {\em Proceedings of the IEEE/CVF Conference on Computer Vision and
  Pattern Recognition}, pages 9243--9252, 2020.

\bibitem{sohl-dickstein2015deep}
Jascha Sohl-Dickstein, Eric~A Weiss, Niru Maheswaranathan, and Surya Ganguli.
\newblock Deep unsupervised learning using nonequilibrium thermodynamics.
\newblock {\em arXiv preprint arXiv:1503.03585}, March 2015.

\bibitem{song2020multi}
Jiaming Song and Stefano Ermon.
\newblock Multi-label contrastive predictive coding.
\newblock {\em arXiv preprint arXiv:2007.09852}, 2020.

\bibitem{song2018learning}
Jiaming Song, Pratyusha Kalluri, Aditya Grover, Shengjia Zhao, and Stefano
  Ermon.
\newblock Learning controllable fair representations.
\newblock {\em arXiv preprint arXiv:1812.04218}, December 2018.

\bibitem{song2020denoising}
Jiaming Song, Chenlin Meng, and Stefano Ermon.
\newblock Denoising diffusion implicit models.
\newblock {\em arXiv preprint arXiv:2010.02502}, 2020.

\bibitem{song2019generative}
Yang Song and Stefano Ermon.
\newblock Generative modeling by estimating gradients of the data distribution.
\newblock {\em arXiv preprint arXiv:1907.05600}, July 2019.

\bibitem{song2020score}
Yang Song, Jascha Sohl-Dickstein, Diederik~P Kingma, Abhishek Kumar, Stefano
  Ermon, and Ben Poole.
\newblock Score-based generative modeling through stochastic differential
  equations.
\newblock {\em arXiv preprint arXiv:2011.13456}, 2020.

\bibitem{takahashi2019variational}
Hiroshi Takahashi, Tomoharu Iwata, Yuki Yamanaka, Masanori Yamada, and Satoshi
  Yagi.
\newblock Variational autoencoder with implicit optimal priors.
\newblock {\em Proceedings of the AAAI Conference on Artificial Intelligence},
  33(01):5066--5073, July 2019.

\bibitem{tolstikhin2017wasserstein}
Ilya Tolstikhin, Olivier Bousquet, Sylvain Gelly, and Bernhard Schoelkopf.
\newblock Wasserstein {Auto-Encoders}.
\newblock {\em arXiv preprint arXiv:1711.01558}, November 2017.

\bibitem{tomczak2016improving}
Jakub~M Tomczak and Max Welling.
\newblock Improving variational auto-encoders using householder flow.
\newblock {\em arXiv preprint arXiv:1611.09630}, 2016.

\bibitem{tomczak2017vae}
Jakub~M Tomczak and Max Welling.
\newblock {VAE} with a {VampPrior}.
\newblock {\em arXiv preprint arXiv:1705.07120}, May 2017.

\bibitem{vahdat2020nvae}
Arash Vahdat and Jan Kautz.
\newblock {NVAE}: A deep hierarchical variational autoencoder.
\newblock {\em arXiv preprint arXiv:2007.03898}, July 2020.

\bibitem{oord2016wavenet}
Aaron van~den Oord, Sander Dieleman, Heiga Zen, Karen Simonyan, Oriol Vinyals,
  Alex Graves, Nal Kalchbrenner, Andrew Senior, and Koray Kavukcuoglu.
\newblock {WaveNet}: A generative model for raw audio.
\newblock {\em arXiv preprint arXiv:1609.03499}, September 2016.

\bibitem{oord2016pixel}
Aaron van~den Oord, Nal Kalchbrenner, and Koray Kavukcuoglu.
\newblock Pixel recurrent neural networks.
\newblock {\em arXiv preprint arXiv:1601.06759}, January 2016.

\bibitem{oord2018representation}
Aaron van~den Oord, Yazhe Li, and Oriol Vinyals.
\newblock Representation learning with contrastive predictive coding.
\newblock {\em arXiv preprint arXiv:1807.03748}, July 2018.

\bibitem{oord2017neural}
Aaron van~den Oord, Oriol Vinyals, and Koray Kavukcuoglu.
\newblock Neural discrete representation learning.
\newblock {\em arXiv preprint arXiv:1711.00937}, November 2017.

\bibitem{wang2018pix2pixHD}
Ting-Chun Wang, Ming-Yu Liu, Jun-Yan Zhu, Andrew Tao, Jan Kautz, and Bryan
  Catanzaro.
\newblock High-resolution image synthesis and semantic manipulation with
  conditional gans.
\newblock In {\em CVPR}, 2018.

\bibitem{xia2021tedigan}
Weihao Xia, Yujiu Yang, Jing-Hao Xue, and Baoyuan Wu.
\newblock {TediGAN}: {Text-Guided} diverse face image generation and
  manipulation.
\newblock 2021.

\bibitem{xie2021adversarial}
Zhe Xie, Chengxuan Liu, Yichi Zhang, Hongtao Lu, Dong Wang, and Yue Ding.
\newblock Adversarial and contrastive variational autoencoder for sequential
  recommendation.
\newblock {\em arXiv preprint arXiv:2103.10693}, March 2021.

\bibitem{xu2020generative}
Yinghao Xu, Yujun Shen, Jiapeng Zhu, Ceyuan Yang, and Bolei Zhou.
\newblock Generative hierarchical features from synthesizing images.
\newblock {\em arXiv e-prints}, pages arXiv--2007, 2020.

\bibitem{you2018graph}
Jiaxuan You, Bowen Liu, Rex Ying, Vijay Pande, and Jure Leskovec.
\newblock Graph convolutional policy network for {Goal-Directed} molecular
  graph generation.
\newblock {\em arXiv preprint arXiv:1806.02473}, June 2018.

\bibitem{yu2017seqgan}
Lantao Yu, Weinan Zhang, Jun Wang, and Yong Yu.
\newblock Seqgan: Sequence generative adversarial nets with policy gradient.
\newblock In {\em Proceedings of the AAAI conference on artificial
  intelligence}, volume~31, 2017.

\bibitem{zhao2017infovae}
Shengjia Zhao, Jiaming Song, and Stefano Ermon.
\newblock {{InfoVAE}}: Information maximizing variational autoencoders.
\newblock {\em arXiv preprint arXiv:1706.02262}, June 2017.

\bibitem{zhao2017towards}
Shengjia Zhao, Jiaming Song, and Stefano Ermon.
\newblock Towards deeper understanding of variational autoencoding models.
\newblock {\em arXiv preprint arXiv:1702.08658}, February 2017.

\bibitem{zhao2018a}
Shengjia Zhao, Jiaming Song, and Stefano Ermon.
\newblock A lagrangian perspective on latent variable generative models.
\newblock In {\em Proc. 34th Conference on Uncertainty in Artificial
  Intelligence}, 2018.

\bibitem{zhu2020domain}
Jiapeng Zhu, Yujun Shen, Deli Zhao, and Bolei Zhou.
\newblock In-domain gan inversion for real image editing.
\newblock In {\em European Conference on Computer Vision}, pages 592--608.
  Springer, 2020.

\bibitem{zhu2019lia}
Jiapeng Zhu, Deli Zhao, Bolei Zhou, and Bo~Zhang.
\newblock Lia: Latently invertible autoencoder with adversarial learning.
\newblock 2019.

\end{thebibliography}

\newpage

\appendix

\section{Additional Details for D2C}

\subsection{Training diffusion models}
\label{app:diffusion-training}

We use the notations in \cite{song2020denoising} to denote the $\alpha$ values and consider the forward diffusion model in \cite{ho2020denoising}; a non-Markovian version that motivates other sampling procedures can be found in \cite{song2020denoising}, but the training procedure is largely identical. We refer to the reader to these two papers for more details.

First, we define the following diffusion forward process for a series $\{\alpha_t\}_{t=0}^{T}$:
\begin{gather}
    q(\rvx^{(\alpha_{1:T})} | \rvx^{(\alpha_{0})}) := \prod_{t=1}^{T} q(\rvx^{(\alpha_t)} | \rvx^{(\alpha_{t-1})}), \\
    q(\rvx^{(\alpha_t)} | \rvx^{(\alpha_{t-1})}) := \gN\left(\sqrt{\frac{\alpha_t}{\alpha_{t-1}}} \rvx^{(\alpha_{t-1})}, \left(1 - \frac{\alpha_t}{\alpha_{t-1}}\right) \mI\right), \label{eq:diff-ho}
\end{gather}
and from standard derivations for Gaussian we have that:
\begin{align}
    q(\rvx^{(\alpha_{t-1})} | \rvx^{(\alpha_{t})}, \rvx^{(\alpha_{0})}) = \gN\Bigg(\underbrace{\frac{\sqrt{\alpha_{t-1} - \alpha_t}}{1 - \alpha_t} \rvx^{(\alpha_{0})} + \frac{\alpha_t (1 - \alpha_{t-1})}{\alpha_{t-1} (1 - \alpha_t)} \rvx^{(\alpha_{t})}}_{ \tilde{\mu}(\rvx^{(\alpha_{t})}, \rvx^{(\alpha_{0})}; \alpha_t, \alpha_{t-1})}, \frac{1 - \alpha_{t-1}}{1 - \alpha_t} \left({1 - \frac{\alpha_t}{\alpha_{t-1}}}\right)\mI\Bigg). \label{eq:q-reverse}
\end{align}
As a variational approximation to the above, \cite{ho2020denoising} considered a specific type of $p_\theta(\rvx^{(\alpha_{t-1})} | \rvx^{(\alpha_{t})})$:
\begin{align}
    p_\theta(\rvx^{(\alpha_{t-1})} | \rvx^{(\alpha_{t})}) = \gN\left({ \mu_\theta(\rvx^{(\alpha_{t})}; \alpha_{t}, \alpha_{t-1})}, (\sigma^{(\alpha_t)})^2 \mI\right), \label{eq:p-diffusion}
\end{align}
where $\mu_\theta$ and $\sigma^{(\alpha_t)}$ are parameters, and we remove the superscript of $\ptheta$ to indicate that there are no additional discretization steps in between (the sampling process is explicitly defined). 
Then, we have the standard variational objective as follows:
\begin{align}
   { L} & := \bb{E}_{q} \left[\log q(\rvx^{(\alpha_{T})} | \rvx^{(\alpha_{0})}) + \sum_{t=2}^{T} \log q(\rvx^{(\alpha_{t-1})} | \rvx^{(\alpha_{t})}, \rvx^{(\alpha_{0})}) - \sum_{t=1}^{T} \log p_\theta^{(\alpha_t, \alpha_{t-1})}(\rvx^{(\alpha_{t-1})} | \rvx^{(\alpha_{t})}) \right] \nonumber \\
   & \equiv \bb{E}_{q} \left[\sum_{t=2}^{T} \underbrace{\KL(q(\rvx^{(\alpha_{t-1})} | \rvx^{(\alpha_{t})}, \rvx^{(\alpha_{0})})) \Vert p_\theta(\rvx^{(\alpha_{t-1})} | \rvx^{(\alpha_{t})}))}_{ L_{t-1}} - \log p_\theta(\rvx^{(\alpha_{0})} | \rvx^{(\alpha_1)}) \right], \nonumber
\end{align}
where $\equiv$ denotes ``equal up to a constant that does not depend on $\theta$'' and each ${ L_{t-1}}$ is a KL divergence between two Gaussian distributions. Let us assume that the standard deviation of $p_\theta(\rvx^{(\alpha_{t-1})} | \rvx^{(\alpha_{t})})$ is equal to that of $q(\rvx^{(\alpha_{t-1})} | \rvx^{(\alpha_{t})}, \rvx^{(\alpha_{0})}))$, which we denote as $\sigma^{(\alpha_t)}$. And thus:
\begin{align}
    { L_{t-1}} = \bb{E}_q\left[\frac{1}{2(\sigma^{(\alpha_t)})^2} \norm{{ \mu_\theta(\rvx^{(\alpha_{t})}; \alpha_{t}, \alpha_{t-1})} -  \tilde{\mu}(\rvx^{(\alpha_{t})}, \rvx^{(\alpha_{0})}; \alpha_t, \alpha_{t-1})}_2^2\right]. \label{eq:deriv-start}
\end{align}
With a particular reparametrization from $\mu_\theta$ to $\epsilon_\theta$ (which tries to model the noise vector at $\alpha_t$):
\begin{align}
    \mu_\theta(\rvx^{(\alpha_{t})}; \alpha_{t}, \alpha_{t-1}) = \sqrt{\frac{\alpha_{t-1}}{\alpha_t}} \left(\rvx^{(\alpha_{t})} - \frac{\sqrt{\alpha_{t-1} - \alpha_t}}{\sqrt{(1 - \alpha_t) \alpha_t}} \cdot \epsilon_\theta(\rvx^{(\alpha_{t})}; \alpha_{t})\right),
\end{align}
the objective function can be simplified to:
\begin{align}
    { L_{t-1}} = \bb{E}_{\rvx_0, \epsilon}\left[\frac{(\alpha_{t-1} - \alpha_t)}{2(\sigma^{(\alpha_t)})^2 (1 - \alpha_t) \alpha_t} \norm{\epsilon - { \epsilon_\theta(\rvx^{(\alpha_{t})}; \alpha_{t}, \alpha_{t-1})}}_2^2\right] \label{eq:deriv-end}
\end{align}
where $\rvx^{(\alpha_{t})} = \sqrt{\alpha_t} \rvx_0 + \sqrt{1 - \alpha_t} \epsilon$. 
Intuitively, this is a weighted sum of mean-square errors between the noise model $\epsilon_\theta$ and the actual noise $\epsilon$. Other weights can also be derived with different forward processes that are non-Markovian~\cite{song2020denoising}, and in practice, setting the weights to 1 is observed to achieve decent performance for image generation.

\subsection{DDIM sampling procedure}
\label{app:ddim}

In this section, we discuss the detailed sampling procedure from $\rvx^{(0)} \sim \gN(0, \mI)$ (which is the distribution with ``all noise''\footnote{Technically, the maximum noise level $\alpha_T$ should have $\alpha_T \to 0$ but not equal to 0, but we can approximate the distribution of $\rvx^{(\alpha_T)}$ with that of $\rvx^{(0)}$ arbitrarily well in practice.}) to $\rvx^{(1)}$ (which is the model distribution with ``no noise''). More specifically, we discuss a deterministic sampling procedure, which casts the generation procedure as an implicit model~\cite{song2020denoising}. Compared to other procedures (such as the one in DDPM~\cite{ho2020denoising}), this has the advantage of better sample quality when few steps are allowed to produce each sample, as well as a near-invertible mapping between $\rvx^{(0)}$ and $\rvx^{(1)}$. We describe this procedure in Algorithm~\ref{alg:ddim-sampling}, where we can choose different series of $\alpha$ to control how many steps (and through which steps) we wish to draw a sample. 
The DDIM sampling procedure corresponds to a particular  discretization to an ODE, we note that it is straightforward to also define the sampling procedure between any two $\alpha$ values. Similarly, given an observation $\rvx^{(1)}$ we can obtain the corresponding latent code $\rvx^{(0)}$ by sampling running Algorithm~\ref{alg:ddim-sampling} with the sequence of $\alpha$ reversed.

\begin{algorithm}[H]
    \small
    \caption{Sampling with the DDIM procedure}
    \label{alg:ddim-sampling}
    \begin{algorithmic}[1]
    \State \textbf{Input}: non-increasing series $\{\alpha_t\}_{t = 0}^{T}$ with $\alpha_T = 0$ and $\alpha_0 = 1$.
    \State Sample $\rvx^{(1)} \sim \gN(0, \mI)$.
    \For{$k \gets T$ to $1$} 
    \State Update $\rvx^{(\alpha_{t-1})}$ from $\rvx^{(\alpha_{t})}$ such that
    $$
    \sqrt{\frac{1}{\alpha_{t-1}}} \rvx^{(\alpha_{t-1})} = \sqrt{\frac{1}{\alpha_t}} \rvx^{(\alpha_t)} + \left(\sqrt{\frac{1 - \alpha_{t-1}}{\alpha_{t-1}}} - \sqrt{\frac{1 - \alpha_{t}}{\alpha_t}}\right) \cdot \epsilon_\theta(\rvx^{(\alpha_{t})}; \alpha_{t})
    $$
    \EndFor
    \State \textbf{Output} $\rvx^{(0)}$.
    \end{algorithmic}
\end{algorithm}

\subsection{Contrastive representation learning}
\label{app:cpc}

\newcommand{\fij}{g(\rvy_i,\overline{\rvw_{i,j}})}
\newcommand{\fjk}{g(\rvy_j,\overline{\rvw_{j,k}})}
\newcommand{\fii}{g(\rvy_i,\rvw_i)}
\newcommand{\fjj}{g(\rvy_j,\rvw_j)}
\newcommand{\xiyi}{(\rvy_i, \rvw_i)}
\newcommand{\xiyj}{(\rvy_i, \overline{\rvw_{i,j}})}
\newcommand{\pxy}[1]{p^{#1}(\rvy, \rvw)}
\newcommand{\py}[1]{p^{#1}(\rvw)}
\newcommand{\px}[1]{p^{#1}(\rvy)}
\newcommand{\pxpy}[1]{p^{#1}(\rvy)p^{#1}(\rvw)}
\newcommand{\rij}{r(\rvw_{i,j})}
\newcommand{\rjk}{r(\rvw_{j,k})}
\newcommand{\rii}{r(\rvy_i)}
\newcommand{\rjj}{r(\rvy_j)}
\newcommand{\aug}{\mathrm{aug}}

In contrastive representation learning, the goal is to distinguish a \textit{positive} pair $(\rvy, \rvw) \sim \pxy{}$ from $(m-1)$ \textit{negative} pairs $(\rvy, \overline{\rvw}) \sim \pxpy{}$. In our context, the positive pairs are representations from the same image, and negative pairs are representations from different images; these images are pre-processed with strong data augmentations~\cite{chen2020a} to encourage rich representations. With two random, independent data augmentation procedures defined as $\aug_1$ and $\aug_2$, we define $\pxy{}$ and $\pxpy{}$ via the following sampling procedure:
\begin{gather*}
    (\rvy, \rvw) \sim \pxy{}: \rvy \sim q_\phi(\vz^{(1)} | \aug_1(\rvx)), \rvw \sim q_\phi(\vz^{(1)} | \aug_2(\rvx)), \rvx \sim \pdata(\rvx), \\
    (\rvy, \rvw) \sim \pxpy{}: \rvy \sim q_\phi(\vz^{(1)} | \aug_1(\rvx_1)), \rvw \sim q_\phi(\vz^{(1)} | \aug_2(\rvx_2)), \rvx_1, \rvx_2 \sim \pdata(\rvx).
\end{gather*}

For a batch of $n$ positive pairs $\{(\rvy_i, \rvw_i)\}_{i=1}^{n}$, the contrastive predictive coding (CPC, \cite{oord2018representation}) objective is defined as:
\begin{align}
    L_{\mathrm{CPC}}(g; q_\phi) := \mathop{\E}\Bigg[\frac{1}{n} \sum_{i=1}^n \log{\frac{m \cdot \fii}{\fii + \sum_{j=1}^{m-1} \fij}}\Bigg]
    \label{eq:cpc}
\end{align}
for some positive critic function $g: \gY \times \gZ \to \R_{+}$, where the expectation is taken over $n$ positive pairs $\xiyi \sim \pxy{}$ and $n (m-1)$ negative pairs $\xiyj \sim \pxpy{}$. Another interpretation to CPC is that it performs $m$-way classification where the ground truth label is assigned to the positive pair.
The representation learner $q_\phi$ then aims to maximize the CPC objective, or to minimize the following objective:
\begin{align}
   - L_{\mathrm{C}}(q_\phi) := \min_{g} - L_{\mathrm{CPC}}(g; q_\phi),
\end{align}
Different specific implementations, such as MoCo~\cite{he2019momentum,chen2020improved,chen2021an} and SimCLR~\cite{chen2020a} can all be treated as specific implementations of this objective function. In this paper, we considered using MoCo-v2~\cite{chen2020a} as our implementation for $L_{\mathrm{C}}$ objective; in principle, other implementations to CPC can also be integrated into D2C as well.

\subsection{Training D2C}
\label{app:d2c-training}

In Algorithm~\ref{alg:d2c-train}, we describe a high-level procedure that trains the D2C model; we note that this procedure does not have any adversarial components. On the high-level, this is the integration of three objectives: the reconstruction objective via the autoencoder, the diffusion objective over the latent space, and the contrastive objective over the latent space. In principle, the \textbf{\color{yellow!50!black}[reconstruction]}, \textbf{\color{blue!50!black}[constrastive]}, and \textbf{\color{teal!50!black}[diffusion]} components can be optimized jointly or separately; we observe that normalizing the latent $\rvz^{(1)}$ with a global mean and standard deviation before applying the diffusion objective helps learning the diffusion model with a fixed $\alpha$ series.

\begin{algorithm}[H]
    \small
    \caption{Training D2C}
    \label{alg:d2c-train}
    \begin{algorithmic}
    \State \textbf{Input}: Data distribution $\pdata$.
    \While{training} \State
    \begin{tcolorbox}[colback=red!5!white,colframe=red!50!black!50!,left=2pt,right=2pt,top=0pt,bottom=1pt,
  colbacktitle=red!25!white,title=\textbf{\color{black} [Draw samples with data augmentation]}]
    \State Draw $m$ samples $\rvx_{0:m-1} \sim \pdata(\rvx)$.
    \State Draw $(m+1)$ data augmentations $\aug_{0}, \ldots \aug_{m-1}$ and $\overline{\aug}_{0}$.
    \For{$i \gets 0$ to $m-1$}
    \State Draw $\rvz_{i}^{(1)} \sim \qphi(\rvz^{(1)} | \aug_i(\rvx))$.
    \EndFor
    \State Draw $\overline{\rvz}_{0}^{(1)} \sim \qphi(\rvz^{(1)} | \overline{\aug}_0(\rvx))$.
    \end{tcolorbox}
    \begin{tcolorbox}[colback=yellow!5!white,colframe=yellow!50!black!50!,left=2pt,right=2pt,top=0pt,bottom=1pt,
  colbacktitle=yellow!25!white,title=\textbf{\color{black} [Reconstruction]}]
    \State Reconstruct $\rvx_0 \sim \ptheta(\rvx | \rvz_0^{(1)})$
    \State Minimize $L_{\mathrm{recon}} = - \log \ptheta(\rvx | \rvz_0^{(1)})$ over $\theta$ and $\phi$ with gradient descent.
    \end{tcolorbox}
    \begin{tcolorbox}[colback=blue!5!white,colframe=blue!50!black!50!,left=2pt,right=2pt,top=0pt,bottom=1pt,
  colbacktitle=blue!25!white,title=\textbf{\color{black} [Contrastive]}]
    \State Define a classification task: assign label 1 to $({\rvz}_{0}^{(1)}, \overline{\rvz}_{0}^{(1)})$ and label 0 to $({\rvz}_{0}^{(1)}, {\rvz}_{i}^{(1)})$ for $i \neq 0$.
    \State Define $L_{\mathrm{CPC}}(g; \qphi)$ as the loss to minimize for the above task, with $g$ as the classifier.
    \State Define $\hat{g}$ as a minimizer to the classifier objective $L_{\mathrm{CPC}}(g; \qphi)$.
    \State Minimize $L_{\mathrm{CPC}}(\hat{g}; \qphi)$ over $\phi$ with gradient descent.
    \end{tcolorbox}
    \begin{tcolorbox}[colback=teal!5!white,colframe=teal!50!black!50!,left=2pt,right=2pt,top=0pt,bottom=1pt,
  colbacktitle=teal!25!white,title=\textbf{\color{black} [Diffusion]}]
    \State Sample $\epsilon \sim \gN(0, I)$, $t \sim \mathrm{Uniform}(1, \ldots, T)$.
    \State Define $\vz_0^{(\alpha_t)} = \sqrt{\alpha_t} \vz^{(0)}_0 + \sqrt{1 - \alpha_t} \epsilon$.
    \State Minimize $\norm{\epsilon - \epsilon_\theta(\vz_0^{(\alpha_t)}; \alpha_t)}_2^2$ over $\theta$ with gradient descent.
    \end{tcolorbox}
    \EndWhile
    \end{algorithmic}
\end{algorithm}

\subsection{Few-shot conditional generation}
\label{app:few-shot}

In order to perform few-shot conditional generation, we need to implement line 4 in Algorithm~\ref{alg:few-shot-examples}, where an unnormalized (energy-based) model is defined over the representations. After we have defined the energy-based model, we implement a procedure to draw samples from this unnormalized model. We note that our approach (marked in teal boxes) is only one way of drawing valid samples, and not necessarily the optimal one. Furthermore, these implementations can also be done over the image space (which is the case for DDIM-I), which may costs more to compute than over the latent space since more layers are needed in a neural network to process it.

For generation from labels, we would define the energy-based model over latents as the product of two components: the first is the ``prior'' over $\rvz^{(1)}$ as defined by the diffusion model and the second is the ``likliehood'' of the label $\rvc$ being true given the latent variable $\rvz^{(1)}$. This places high energy values to the latent variables that are likely to occur under the diffusion prior (so generated images are likely to have high quality) as well as latent variables that have the label $\rvc$. To sample from this energy-based model, we perform a rejection sampling procedure, where we reject latent samples from the diffusion model that have low discrminator values. This procedure is describe in Algorithm~\ref{alg:few-shot-labels}.

\begin{center}
\begin{minipage}{0.5\textwidth}
\begin{algorithm}[H]
\small
    \caption{Generate from labels}
    \label{alg:few-shot-labels}
    \begin{algorithmic}
    \State \textbf{Input} model $r_\psi(\rvc | \rvz^{(1)})$, target label $\rvc$.
    \begin{tcolorbox}[colback=teal!5!white,colframe=teal!50!black!50!,left=2pt,right=2pt,top=0pt,bottom=1pt,
  colbacktitle=teal!25!white,title=\textbf{\color{black} Define latent energy-based model}]
    $E(\hat{\rvz}^{(1)}) = r_\psi(\rvc | \hat{\rvz}^{(1)}) \cdot p_\theta^{(1)}(\hat{\rvz}^{(1)})$
    \end{tcolorbox}
    \begin{tcolorbox}[colback=teal!5!white,colframe=teal!50!black!50!,left=2pt,right=2pt,top=0pt,bottom=1pt,
  colbacktitle=teal!25!white,title=\textbf{\color{black} Sample from $E(\hat{\rvz}^{(1)})$}]
    \While{True}
    \State Sample $\hat{\rvz}^{(1)} \sim p_\theta^{(1)}(\hat{\rvz}^{(1)})$;
    \State Sample $u \sim \mathrm{Uniform}(0, 1)$;
    \State \textbf{If} $u < r_\psi(\rvc | \hat{\rvz}^{(1)})$ \textbf{then} break.
    \EndWhile
    \end{tcolorbox}
    \State \textbf{Output} $\hat{\rvx} \sim \ptheta(\rvx | \hat{\rvz}^{(1)})$.
    \end{algorithmic}
\end{algorithm}
\end{minipage}
\end{center}

For generation from manipulation constraints, we need to further define a prior that favors closeness to the given latent variable so that the manipulated generation is close to the given image except for the label $\rvz$. If the latent variable for the original image is $\rvz^{(1)} \sim \qphi(\rvz^{(1)} | \rvx)$, then we define the closeness via the L2 distance between the it and the manipulated latent. We obtain the energy-based model by multiplying this with the diffusion ``prior'' and the classifier ``likelihood''. Then, we approximately draw samples from this energy by taking a gradient step from the original latent value $\rvz^{(1)}$ and then regularizing it with the diffusion prior; this is described in Algorithm~\ref{alg:few-shot-manip}. A step size $\eta$, diffusion noise magnitude $\alpha$, and the diffusion steps from $\alpha$ to $1$ are chosen as hyperparameters. We choose one $\eta$ for each attribute, $\alpha \approx 0.9$, and number of discretization steps to be $5$\footnote{The results are not particularly sensitive to how the discretization steps are chosen. For example, one can take $0.9 \to 0.92 \to 0.96 \to 0.98 \to 0.99 \to 1$.}; we tried $\alpha \in [0.65, 0.9]$ and found that our results are not very sensitive to values within this range. We list the $\eta$ values for each attribute (details in Appendix~\ref{app:exp}).

We note that a more principled approach is to take gradient with respect to the entire energy function (\textit{e.g.}, for Langevin dynamics), where the gradient over the DDIM can be computed with instantaneous change-of-variables formula~\cite{chen2018neural}; we observe that our current version is computationally efficient enough to perform well.

\begin{center}
\begin{minipage}{0.8\textwidth}
\begin{algorithm}[H]
\small
    \caption{Generate from manipulation constraints}
    \label{alg:few-shot-manip}
    \begin{algorithmic}
     \State \textbf{Input} model $r_\psi(\rvc | \rvz^{(1)})$, target label $\rvc$, original image $\rvx$.
    \State Acquire latent $\rvz^{(1)} \sim \qphi(\rvz^{(1)} | \rvx)$;
    \State Fit a model $r_\psi(\rvc | \rvz^{(1)})$ over $\{(\rvz^{(1)}_i, \rvc_i)\}_{i=1}^{n}$
    \begin{tcolorbox}[colback=teal!5!white,colframe=teal!50!black!50!,left=2pt,right=2pt,top=0pt,bottom=1pt,
  colbacktitle=teal!25!white,title=\textbf{\color{black} Define latent energy-based model}]
    $$E(\hat{\rvz}^{(1)}) = r_\psi(\rvc | \hat{\rvz}^{(1)}) \cdot p_\theta^{(1)}(\hat{\rvz}^{(1)}) \cdot \norm{\rvz^{(1)} - \hat{\rvz}^{(1)}}_2^2$$
    \end{tcolorbox}
    \begin{tcolorbox}[colback=teal!5!white,colframe=teal!50!black!50!,left=2pt,right=2pt,top=0pt,bottom=1pt,
  colbacktitle=teal!25!white,title=\textbf{\color{black} Sample from $E(\hat{\rvz}^{(1)})$ (approximate)}]
  \State Choose hyperparameters $\eta > 0, \alpha \in (0, 1)$.
    \State Take a gradient step $\bar{\rvz}^{(1)} \gets \rvz^{(1)} + \eta \nabla_{\rvz} r_\psi(\rvc | \rvz) \vert_{\rvz = \rvz^{(1)}}$.
    \State Add noise $\tilde{\rvz}^{(\alpha)} \gets \sqrt{\alpha} \bar{\rvz}^{(1)} + \sqrt{1 - \alpha} \epsilon$.
    \State Sample $\hat{\rvz}^{(1)} \sim p_\theta^{(\alpha, 1)}(\rvz^{(1)} | \tilde{\rvz}^{(\alpha)})$ with DDIM, \textit{i.e.}, use the diffusion prior to ``denoise''.
    \end{tcolorbox}
    \State \textbf{Output} $\hat{\rvx} \sim \ptheta(\rvx | \hat{\rvz}^{(1)})$.
    \end{algorithmic}
\end{algorithm}
\end{minipage} 
\end{center}

\section{Formal Statements and Proofs}

\subsection{Relationship to maximum likelihood}
\label{app:mle}

\mleinf*

\begin{theorem}(formal)
Suppose that $\rvx \in \R^d$. For any valid $\{\alpha_i\}_{i=0}^{T}$, let $\hat{w}$ satisfy:
\begin{gather}
   \forall t \in [2, \ldots, T], \quad \hat{w}(\alpha_t) = \frac{(1 - \alpha_t) \alpha_{t-1}}{2 (1 - \alpha_{t-1})^2 \alpha_t} \\
   \hat{w}(\alpha_1) = \frac{1 - \alpha_1}{2 (2 \pi)^d \alpha_1}
\end{gather}
then:
\begin{align}
    - L_{\mathrm{D2}}(\theta, \phi; \hat{w}) + H(\qphi(\rvz^{(1)} | \rvx)) \leq \bb{E}_{\pdata}[\log \ptheta(\rvx)]
\end{align}
where $\ptheta(\rvx) := \bb{E}_{\rvx_0 \sim p^{(0)}(\rvz^{(0)})}[\ptheta(\rvx | \rvz^{(0)})]$ is the marginal probability of $\rvx$ under the D2C model.
\end{theorem}
\begin{proof}
First, we have that:
\begin{align}
    \bb{E}_{\pdata(\rvx)}[\log \ptheta(\rvx)] & = \bb{E}_{\pdata(\rvx)}\left[\log \sum_{\rvz^{(1)}} \ptheta(\rvx | \rvz^{(1)}) \ptheta(\rvz^{(1)})\right] \\
    & \geq \bb{E}_{\pdata(\rvx), \qphi(\rvz^{(1)})}[\log \ptheta(\rvx | \rvz^{(1)}) + \log \ptheta(\rvz^{(1)}) - \log \qphi(\rvz^{(1)} | \rvx)] \\
     & = \bb{E}_{\pdata(\rvx), \qphi(\rvz^{(1)} | \rvx)}[\log \ptheta(\rvx | \rvz^{(1)}) - \KL(\qphi(\rvz^{(1)} | \rvx) \Vert \ptheta(\rvz^{(1)})) ].
\end{align}
where we use Jensen's inequality here. Compared with the objective for D2:
\begin{align}
    - L_\mathrm{D2}(\theta, \phi; w) & := \E_{\rvx \sim \pdata, \rvz^{(1)} \sim \qphi(\rvz^{(1)} | \rvx)}[\log p(\rvx | \rvz^{(1)}) - \ell_{\text{diff}}(\rvz^{(1)}; w, \theta)],
\end{align}
and it is clear the proof is complete if we show that:
\begin{align}
    & H(\qphi(\rvz^{(1)} | \rvx)) - \bb{E}_{\rvz^{(1)} \sim \qphi(\rvz^{(1)} | \rvx)}[\ell_{\text{diff}}(\rvz^{(1)}; \hat{w}, \theta)]   \\ 
    \leq  & -\KL(\qphi(\rvz^{(1)} | \rvx) \Vert \ptheta(\rvz^{(1)})) \\ 
    = &  H(\qphi(\rvz^{(1)} | \rvx)) + \bb{E}_{\rvz^{(1)} \sim \qphi(\rvz^{(1)} | \rvx)}[\log p_\theta(\rvz^{(1)})]
\end{align}
or equivalently:
\begin{align}
    \bb{E}_{\rvz^{(1)} \sim \qphi(\rvz^{(1)} | \rvx)}[\ell_{\text{diff}}(\rvz^{(1)}; \hat{w}, \theta)] \leq \bb{E}_{\rvz^{(1)} \sim \qphi(\rvz^{(1)} | \rvx)}[\log p_\theta(\rvz^{(1)})]
\end{align}
Let us apply variational inference with an inference model $q(\rvz^{(\alpha_{1:T})}|\rvz^{(1)})$ where $\alpha_0 = 1$ and $\alpha_T = 0$:
\begin{align}
    & \bb{E}_{\rvz^{(1)} \sim \qphi(\rvz^{(1)} | \rvx)}[\log p_\theta(\rvz^{(1)})] =  \bb{E}_{\rvz^{(1)} \sim \qphi(\rvz^{(1)} | \rvx)}[\log \sum_{\rvz} \big( p_\theta(\rvz^{(\alpha_{T})}) \prod_{t=1}^{T} p_\theta(\rvz^{(\alpha_{t-1})} | \rvz^{(\alpha_t)}) \big)] \nonumber \\
    \geq \ & \bb{E}_{\rvz^{(\alpha_{0:T})}}[  \log p_\theta(\rvz^{(\alpha_{T})}) + \sum_{t=1}^{T} \log p_\theta(\rvz^{(\alpha_{t-1})} | \rvz^{(\alpha_t)}) - \log q(\rvz^{(\alpha_{1:T})}|\rvz^{(\alpha_0)}) ] \\
    \geq \ & \bb{E}_{\rvz^{(\alpha_{0:T})}}\Big[  \log p_\theta(\rvz^{(\alpha_{T})}) -  \log q(\rvz^{(\alpha_{T})} | \rvz^{(\alpha_{0})})  \\ 
    & \qquad \qquad \qquad - \sum_{t=2}^{T} \underbrace{\KL(q(\rvz^{(\alpha_{t-1})} | \rvz^{(\alpha_t)}, \rvz^{(\alpha_0)}) \Vert  p_\theta(\rvz^{(\alpha_{t-1})} | \rvz^{(\alpha_t)}))}_{L_{t-1}} + \log p_\theta(\rvz^{(\alpha_{0})} | \rvz^{(\alpha_1)})  \Big] \nonumber
\end{align}
where we remove the superscript of $p_\theta$ to indicate that there are no intermediate discretization steps between $\alpha_{t-1}$ and $\alpha_t$. Now, for $t \geq 2$, let us consider $\ptheta$ and $\qphi$ with the form in Equations~\ref{eq:q-reverse} and \ref{eq:p-diffusion} respectively, which are both Gaussian distributions (restrictions to $\ptheta$ will still give lower bounds). Then we can model the standard deviation of $p_\theta(\rvx^{(\alpha_{t-1})} | \rvx^{(\alpha_{t})})$ to be equal to that of $q(\rvx^{(\alpha_{t-1})} | \rvx^{(\alpha_{t})}, \rvx^{(\alpha_{0})}))$. Under this formulation, the KL divergence for $L_{t-1}$ is just one between two Gaussians with the same standard deviations and is a weighted Euclidean distance between the means. Using the derivation from \Eqref{eq:deriv-start} to \Eqref{eq:deriv-end}, we have that:
\begin{align}
    L_{t-1} = \bb{E}_{\rvz_0, \epsilon}\left[\frac{(1 - \alpha_t) \alpha_{t-1}}{2 (1 - \alpha_{t-1})^2 \alpha_t} \norm{\epsilon - { \epsilon_\theta(\rvz^{(\alpha_{t})}; \alpha_{t}, \alpha_{t-1})}}_2^2\right]
\end{align}
which gives us the weights for $\hat{w}$ for $\alpha_{2:T}$. For $p_\theta(\rvz^{(\alpha_{0})} | \rvz^{(\alpha_1)})$ let us model it to be a Gaussian with mean
$$
\mu_\theta(\rvz^{(\alpha_1)}; \alpha_1, \alpha_0) = \frac{\rvz^{(\alpha_1)} - \sqrt{1 - \alpha_t} \epsilon_\theta(\rvz^{(\alpha_{1})}; \alpha_{1}, \alpha_{0})}{\sqrt{\alpha_1}}
$$
and standard deviation $1/\sqrt{2\pi}$ (chosen such that normalization constant is 1). Thus, with $$\rvz^{(0)} = \frac{\rvz^{(\alpha_1)} - \sqrt{1 - \alpha_t} \epsilon}{\sqrt{\alpha_1}}$$ we have that:
\begin{align}
    \log p_\theta(\rvz^{(\alpha_{0})} | \rvz^{(\alpha_1)}) = \frac{1 - \alpha_1}{2 (2 \pi)^d \alpha_1} \norm{\epsilon - { \epsilon_\theta(\rvz^{(\alpha_{1})}; \alpha_{1}, \alpha_{0})}}_2^2
\end{align}
which gives us the weight of $\hat{w}$ for $\alpha_1$. Furthermore:
\begin{align}
    \bb{E}_{\rvz^{(\alpha_{0:T})}}[  \log p_\theta(\rvz^{(\alpha_{T})}) - q(\rvz^{(\alpha_{T})} | \rvz^{(\alpha_{0})})] = 0
\end{align}
because $\rvz^{(\alpha_{T})} \sim \gN(0, \mI)$ for both $\ptheta$ and $q$. Therefore, we have that:
\begin{align}
    \bb{E}_{\rvz^{(1)} \sim \qphi(\rvz^{(1)} | \rvx)}[\ell_{\text{diff}}(\rvz^{(1)}; \hat{w}, \theta)] \leq \bb{E}_{\rvz^{(1)} \sim \qphi(\rvz^{(1)} | \rvx)}[\log p_\theta(\rvz^{(1)})]
\end{align}
which completes the proof.
\end{proof}

\subsection{D2 models address latent posterior mismatch in VAEs}
\label{app:prior-hole}

\priorholeinf*

\begin{theorem}(formal)
Let $\ptheta(\rvz) = \gN(0, \mI)$ where $\rvz \in \R^d$. For any $\epsilon > 0, \delta < 0.5$, there exists a distribution $\qphi(\rvz)$ with an $(\epsilon, \delta)$-prior hole, such that $\KL(\qphi \Vert \ptheta) \leq \log 2$ and $W_2(\qphi, \ptheta) < \gamma$ for any $\gamma > 0$, where $W_2$ is the 2-Wasserstein distance. 
\end{theorem}
\begin{proof}
    Let us define a function $f: \R_{\geq 0} \to [0, 1]$ such that for any Euclidean ball $B(0, R)$ centered at 0 with radius $R$:
    \begin{align}
        f(R) := \int_{B(0, R)} \ptheta(\rvz) \diff \rvz,
    \end{align}
    \textit{i.e.}, $f(R)$ measures the probability mass of the Gaussian distribution $\ptheta(\rvz)$ within $B(0, R)$. As $\mathrm{d}f / \mathrm{d}R > 0$ for $R > 0$, $f$ is invertible.
    
    \begin{figure}[h]
        \centering
        \includegraphics[width=0.5\textwidth]{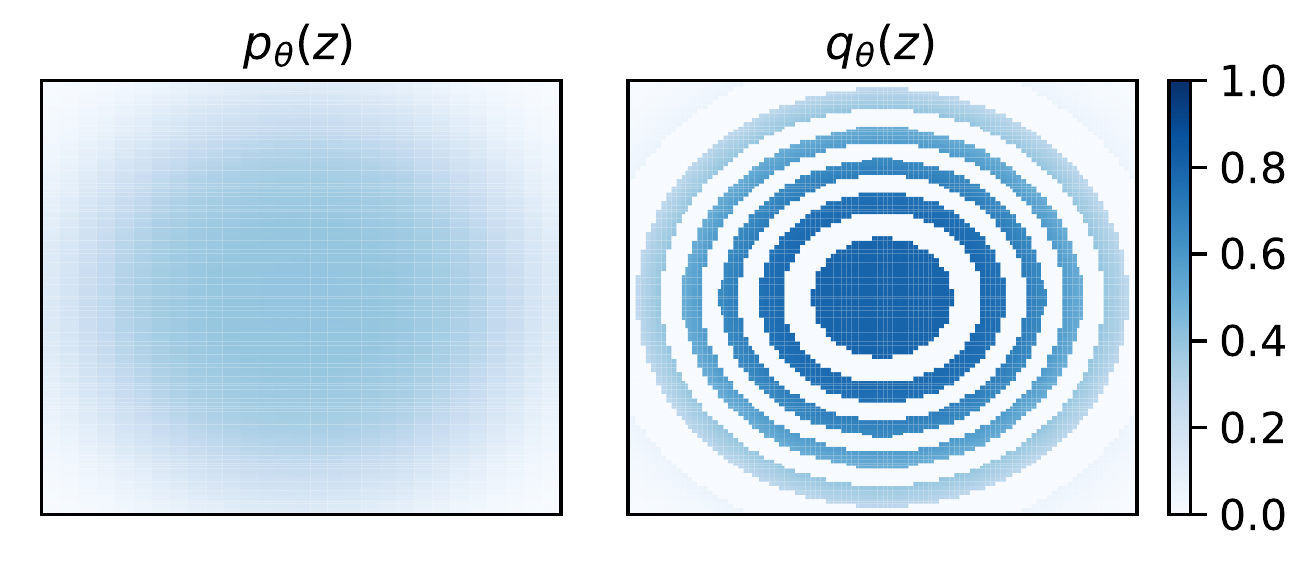}
        \caption{Illustration of the construction in 2d. When we use more rings, the prior hole and upper bound of KL divergence are constant but the upper bound of Wasserstein distance decreases.}
    \end{figure}
    
    Now we shall construct $\qphi(\rvz)$. First, let $\qphi(\rvz) = \ptheta(\rvz)$ whenever $\norm{\rvz}_2 \geq f^{-1}(2\delta)$; then for any $n$, we can find a sequence $\{r_0, r_1, \ldots, r_{2n}\}$ such that:
    \begin{align}
        r_0 = 0, \quad r_{2n} = f^{-1}(2\delta), \quad f(r_i) - f(r_{i-1}) = f^{-1}(2\delta) \delta / n \ \text{for all } k \in \{1, \ldots, 2n\},
    \end{align}
    Intuitively, we find $2n$ circles with radii $\{r_0, \ldots, r_{2n}\}$ whose masses measured by $\ptheta(\rvz)$ is an arithmetic progression $\{0, \delta / 2n, \ldots, 2 \delta \}$.
    We then define $\qphi(\rvz)$ for $\norm{\rvz} < f^{-1}(2\delta)$ as follows:
    \begin{align}
        \qphi(\rvz) = \begin{cases}
        2 \cdot \ptheta(\rvz) & \text{if } \norm{\rvz} \in \bigcup_{k=0}^{n-1} [r_{2k}, r_{2k+1}) \\
        0 & \text{otherwise}
        \end{cases}
    \end{align}
    Intuitively, $\qphi$ is defined by moving all the mass from ring $(2k+1)$ to ring $2k$. Note that this $\qphi(\rvz)$ is a valid probability distribution because:
    \begin{align}
        \int_{\R^d} \qphi(\rvz) \diff \rvz & = \int_{B(0, f^{-1}(2\delta))} \qphi(\rvz) \diff \rvz + \int_{B^{c}(0, f^{-1}(2\delta))} \qphi(\rvz) \diff \rvz \\
        & = 2 \int_{B(0, f^{-1}(2\delta))} \ptheta(\rvz) \mathbb{I}\left(\norm{\rvz} \in \bigcup_{i=0}^{n-1} [r_{2k}, r_{2k+1})\right) \diff \rvz + \int_{B^{c}(0, f^{-1}(2\delta))} \ptheta(\rvz) \diff \rvz \\
        & = \int_{B(0, f^{-1}(2\delta))} \ptheta(\rvz) \diff \rvz + \int_{B^{c}(0, f^{-1}(2\delta))} \ptheta(\rvz) \diff \rvz = 1
    \end{align}
    Next, we validate that $\qphi$ satisfies our constraints in the statement.
    
    \paragraph{Prior hole} Apparently, if we choose $\gS = \bigcup_{k=0}^{n-1} [r_{2k+1}, r_{2k+2})$, then $\int_\gS \ptheta(\rvz) \diff \rvz = \delta$ and $\int_\gS \qphi(\rvz) \diff \rvz = 0$; so $\gS$ instantiates a $(\epsilon, \delta)$-prior hole.
    
    \paragraph{KL divergence} We note that $\qphi(\rvz) \leq 2 \ptheta(\rvz)$ is true for all $\rvz$, so $$\KL(\qphi(\rvz) \Vert \ptheta(\rvz)) = \bb{E}_{\rvz \sim \qphi(\rvz)}[\log \qphi(\rvz) - \log \ptheta(\rvz)] \leq \log 2.$$
    
    \paragraph{2 Wasserstein Distance} We use the Monge formulation:
    \begin{align}
        W_2(\qphi(\rvz), 2 \ptheta(\rvz)) = \min_{T: \qphi = T_\sharp \ptheta} \int_{R^d} \norm{\rvz - T(\rvz)}_2^2 \ptheta(\rvz) \diff \rvz
    \end{align}
    where $T$ is any transport map from $\ptheta$ to $\qphi$. Consider the transport map $\hat{T}$ such that:
    \begin{align}
        \hat{T}(\rvz) = \begin{cases}
        \rvz & \text{if } \quad \qphi(\rvz) \geq 0 \\
        \rvz \cdot f^{-1}(f(\norm{\rvz}) - f(r_{2k+1}) + k \delta / n)& \text{otherwise, for } k \text{ such that }\norm{\rvz}_{2} \leq [r_{2k+1}, r_{2k+2})
        \end{cases}
    \end{align}
    which moves the mass in $[r_{2k+1}, r_{2k+2})$ to $[r_{2k}, r_{2k+1})$. From this definition, we have that $\norm{\hat{T}(\rvz) - \rvz}_2 \leq \max_{k \in \{0, \ldots, n-1\}} (r_{2k+2} - r_{2k})$. Moreover, since by definition,
    \begin{align}
        2\delta/n & = \int_{B(0, r_{2k+2})} p_\theta(\rvz) \diff \rvz - \int_{B(0, r_{2k})} p_\theta(\rvz) \diff \rvz \\
        & > \pi (r_{2k+2}^2 - r_{2k}^2) \min_{\rvz: \norm{\rvz} \in [r_{2k}, r_{2k+2})} p_\theta(\rvz) \\
        & > \pi (r_{2k+2} - r_{2k})^2 \min_{\rvz: \norm{\rvz} \in [r_{2k}, r_{2k+2})} p_\theta(\rvz)
    \end{align}
    We have that
    \begin{align}
        W_2(\qphi(\rvz), 2 \ptheta(\rvz)) & \leq \max_{k \in \{0, \ldots, n-1\}} (r_{2k+2} - r_{2k})^2  < \frac{2\delta}{\pi n \min_{\rvz: \norm{\rvz}_2 \leq r_{2n}} \ptheta(\rvz)} \\
        & <  \frac{2\delta}{\pi n \min_{\rvz: \norm{\rvz}_2 \leq r_{2n}} \ptheta(f^{-1}(2\delta) \rvn)}
    \end{align}
    for any vector $\rvn$ with norm 1. Note that the above inequality is inversely proportional to $n$, which can be any integer. Therefore, for a fixed $\delta$, $W_2(\qphi(\rvz), 2 \ptheta(\rvz)) = O(1/n)$; so for any $\gamma$, there exists $n$ such that $W_2(\qphi(\rvz), 2 \ptheta(\rvz)) < \gamma$, completing the proof.
\end{proof}

\paragraph{Note on DDIM prior preventing the prior hole} For a noise level $\alpha$, we have that:
\begin{align}
    q^{(\alpha)}(\rvz^{(\alpha)}) = \bb{E}_{\rvz^{(1)} \sim q^{(1)}(\rvz^{(1)}}[\gN(\sqrt{\alpha} \rvz^{(1)}, (1 - \alpha)\mI)]
\end{align}
as $\alpha \to 0$, $\KL(q^{(\alpha)}(\rvz^{(\alpha)}) \Vert \gN(0, \mI)) \to 0$. From Pinsker's inequality and the definition of $(\epsilon, \delta)$-prior hole:
\begin{align}
    \delta - \epsilon \leq D_{\mathrm{TV}}(q^{(\alpha)}(\rvz^{(\alpha)}), \gN(0, \mI))) \leq \sqrt{\frac12 \KL(q^{(\alpha)}(\rvz^{(\alpha)}) \Vert \gN(0, \mI))},
\end{align}
we should not expect to see any $(\epsilon, \delta)$-prior hole where the difference between $\delta$ and $\epsilon$ is large.

\section{Experimental details}
\label{app:exp}

\subsection{Architecture details and hyperparameters used for training}
We modify the NVAE~\cite{vahdat2020nvae} architecture 
by removing the ``Combiner Cells'' in both encoder and decoder. For the diffusion model, we use the same architecture with different number of channel multiplications, as used in~\cite{ho2020denoising,song2020denoising}. 
For Contrastive learning, we use the MoCo-v2~\cite{chen2020improved} algorithm with augmentations such as \textit{RandomResizedCrop, ColorJitter, RandomGrayscale, RandomHorizontalFlip}.

Additional details about the hyperparameters used are provided in Table~\ref{tab:hparams}.


\begin{table}[]
\centering
\caption{Hyperparameters across different datasets
}
\label{tab:hparams}

\scalebox{0.75}{
\begin{tabular}{@{}lllllll@{}}
\toprule
\textbf{Hyperparameter}                 & \multicolumn{1}{c}{\begin{tabular}[c]{@{}c@{}}\textbf{CIFAR-10}\\ 32x32\end{tabular}} & \multicolumn{1}{c}{\begin{tabular}[c]{@{}c@{}}\textbf{CIFAR-100}\\ 32x32\end{tabular}} & \multicolumn{1}{c}{\begin{tabular}[c]{@{}c@{}}\textbf{CelebA-64}\\ 64x64\end{tabular}} & \multicolumn{1}{c}{\begin{tabular}[c]{@{}c@{}}\textbf{fMoW}\\ 64x64\end{tabular}} & \multicolumn{1}{c}{\begin{tabular}[c]{@{}c@{}}\textbf{CelebA-HQ-256}\\ 256x256\end{tabular}} & \multicolumn{1}{c}{\begin{tabular}[c]{@{}c@{}}\textbf{FFHQ-256}\\ 256x256\end{tabular}} \\ \midrule
\# of epochs                            & 1000                                                                         & 1000                                                                         & 300                                                                           & 300                                                                      & 200                                                                                 & 100                                                                            \\ \midrule
batch size per GPU                      & 32                                                                           & 32                                                                           & 16                                                                            & 16                                                                       & 3                                                                                   & 3                                                                              \\ \midrule
\# initial channels in enc,             & 128                                                                          & 128                                                                          & 64                                                                            & 64                                                                       & 24                                                                                  & 24                                                                             \\ \midrule
spatial dims of z                       & 16*16                                                                        & 16*16                                                                        & 32*32                                                                         & 32*32                                                                    & 64*64                                                                               & 64*64                                                                          \\ \midrule
\# channel in z                         & 8                                                                            & 8                                                                            & 5                                                                             & 5                                                                        & 8                                                                                   & 8                                                                              \\ \midrule
MoCo-v2 queue size                      & 65536                                                                        & 65536                                                                        & 65536                                                                         & 65536                                                                    & 15000                                                                               & 15000                                                                          \\ \midrule
Diffusion feature map res. & 16,8,4,2                                                                     & 16,8,4,2                                                                     & 32,16,8,4,1                                                                   & 32,16,8,4,1                                                              & 64,32,16,8,2                                                                        & 64,32,16,8,2                                                                   \\ \midrule
$\lambda^{-1}$                                  & 17500                                                                        & 17500                                                                        & 17500                                                                         & 17500                                                                    & 17500                                                                               & 17500                                                                          \\ \midrule
learning rate                           & 0.001                                                                        & 0.001                                                                        & 0.001                                                                         & 0.001                                                                    & 0.001                                                                               & 0.001                                                                          \\ \midrule
Optimizer                               & AdamW                                                                        & AdamW                                                                        & AdamW                                                                         & AdamW                                                                    & AdamW                                                                               & AdamW                                                                          \\ \midrule
\# GPUs                                 & 8                                                                            & 8                                                                            & 4                                                                             & 4                                                                        & 8                                                                                   & 8                                                                              \\ \midrule
GPU Type                                & 16 GB V100                                                                & 16 GB V100                                                                 & 12 GB Titan X                                                                 & 12 GB Titan X                                                            & 16 GB V100                                                                          & 16 GB V100                                                                     \\ \midrule
Total training time (h)                 & 24                                                                           & 24                                                                           & 120                                                                           & 120                                                                      & 96                                                                                  & 96                                                                             \\ \bottomrule
\end{tabular}}
\end{table}

\subsection{Additional details for conditional generation}
For $r_\psi(\rvc | \rvz^{(1)})$ we consider training a linear model over the latent space, which has the advantage of being computationally efficient. For conditional generation on labels, we reject samples if their classifier return are lower than a certain threshold (we used $0.5$ for all our experiments). For conditional image manipulation, we consider the same step size $\eta$ for each attribute: $\eta = 10$ for \textit{red lipstick} and $\eta = 15$ for \textit{blond}. We note that these values are not necessarily the optimal ones, as the intensity of the change can grow with a choice of larger $\eta$ values.

\subsection{Amazon Mechanical Turk procedure}

The mechanical turk evaluation is done for different attributes to find out how evaluators evaluate the different approaches. The evaluators are asked to compare a pair of images, and find the best image, which retains the identity as well as contains the desired attribute. Figure~\ref{fig:amt} a) shows the instructions that was given to the evaluators before starting the test and Figure~\ref{fig:amt} b) contains the UI shown to the evaluators when doing comparison. Each evaluation task contains 10 pairwise comparisons, and we perform 15 such evaluation tasks for each attribute. The reward per task is kept as 0.25\$. Since each task takes around 2.5 mins, so the hourly wage comes to be 6\$ per hour.

\begin{figure}
    \centering
    \includegraphics[width=0.9\textwidth]{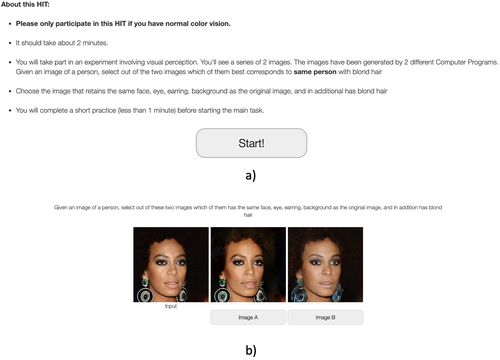}
    \caption{a) Instructions shown to human evaluators for Amazon Mechanical Turk for blond hair before starting the evaluation and b) UI shown to the evaluators when doing comparison.}
    \label{fig:amt}
\end{figure}



\section{Additional Results}
\label{app:results}

We list results for unconditional CIFAR-10 image generation for various types of generative models in Table~\ref{tab:cifar10}. While our results are slightly worse than state-of-the-art diffusion models, we note that our D2C models are trained with relatively fewer resources that some of the baselines; for example, our D2C models is trained on 8 GPUs for 24 hours, whereas NVAE is trained on 8 GPUs for 100 hours and DDPM is trained on v3-8 TPUs for 24 hours. We also note that these comparisons are not necessarily fair in terms of the architecture and compute used to produce the samples.

We list additional image generation results in Figure~\ref{fig:ffhq-long} (unconditional), Figures~\ref{fig:blond-large},~\ref{fig:red-lipstick-large},~\ref{fig:beard-large}, and~\ref{fig:gender-large} (conditional on manipulation constraints), and Figures~\ref{fig:blond_our}, ~\ref{fig:blond_ddim}, ~\ref{fig:gender_our}, and~\ref{fig:gender_ddim} (conditional on labels)\footnote{We will list more results online after publication.}.

\begin{wraptable}{r}{0.45\textwidth}
\vspace{-4em}
    \centering
    \caption{CIFAR-10 image generation results. }
    \vspace{0.5em}
    \label{tab:cifar10}
    \begin{tabular}{l|r}
    \toprule
        Method &  FID\\\midrule
       NVAE~\cite{vahdat2020nvae}  & 51.71 \\
       NCP-VAE~\cite{aneja2020ncp} & 24.08 \\
       EBM~\cite{du2019implicit} & 40.58 \\
       StyleGAN2~\cite{karras2020analyzing} & 3.26 \\
       DDPM~\cite{ho2020denoising} & 3.17 \\
       DDIM~\cite{song2020denoising} & 4.04 \\
       D2C & 10.15 \\\bottomrule
    \end{tabular}
\end{wraptable}

\begin{figure}
    \centering
    \includegraphics[width=\textwidth]{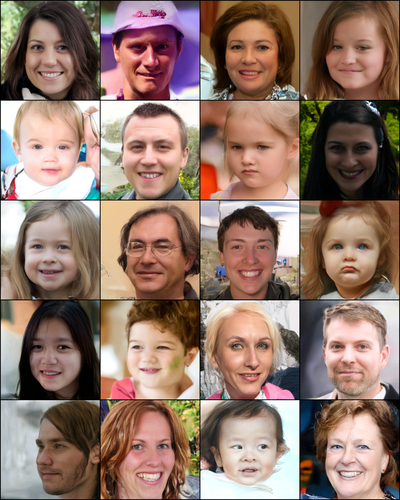}
    \caption{Additional image samples for the FFHQ-256 dataset.}
    \label{fig:ffhq-long}
\end{figure}

\begin{figure}
    \centering
    \includegraphics[width=\textwidth]{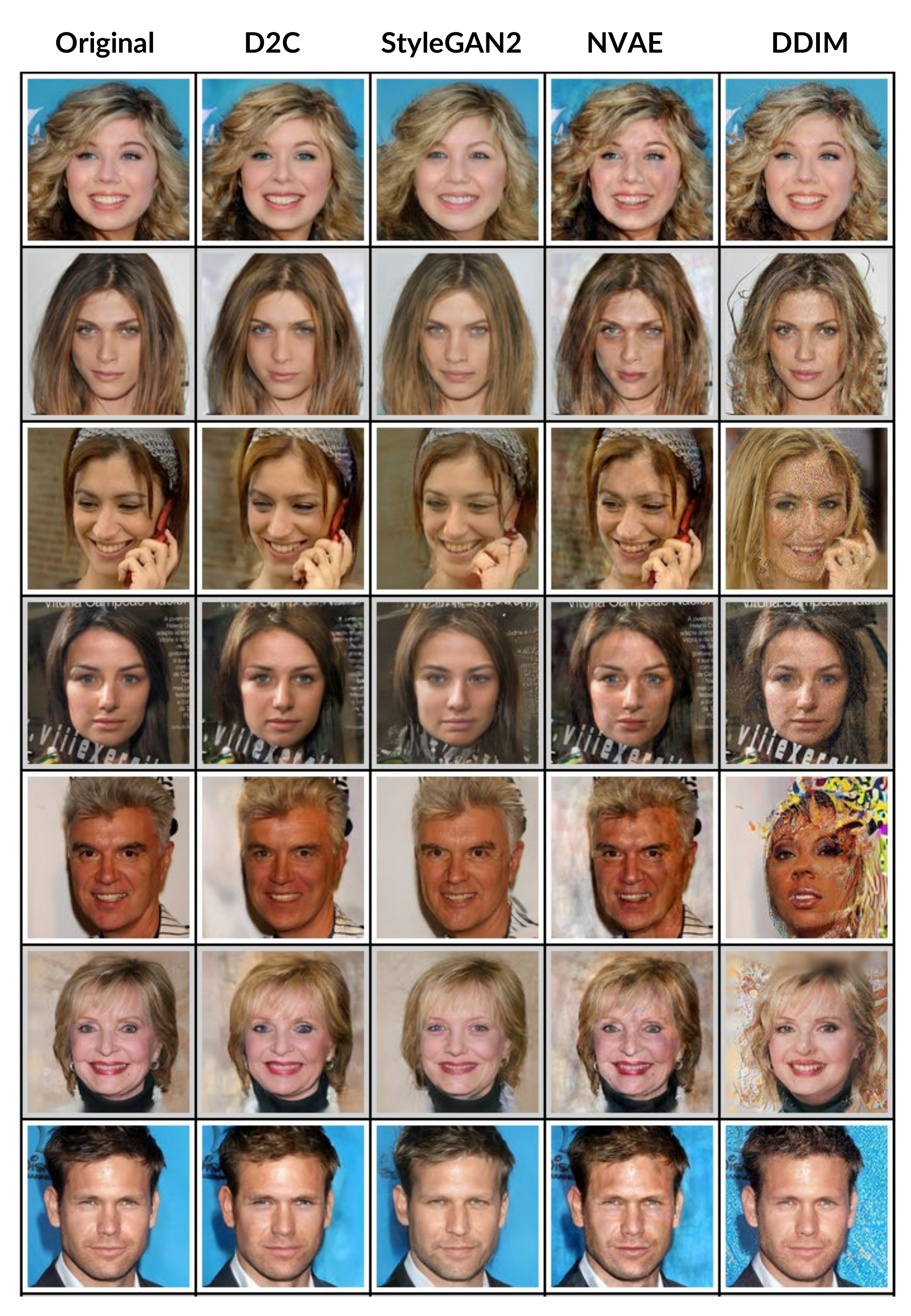}
    \caption{Image manipulation results for \textit{blond hair}.}
    \label{fig:blond-large}
\end{figure}

\begin{figure}
    \centering
    \includegraphics[width=\textwidth]{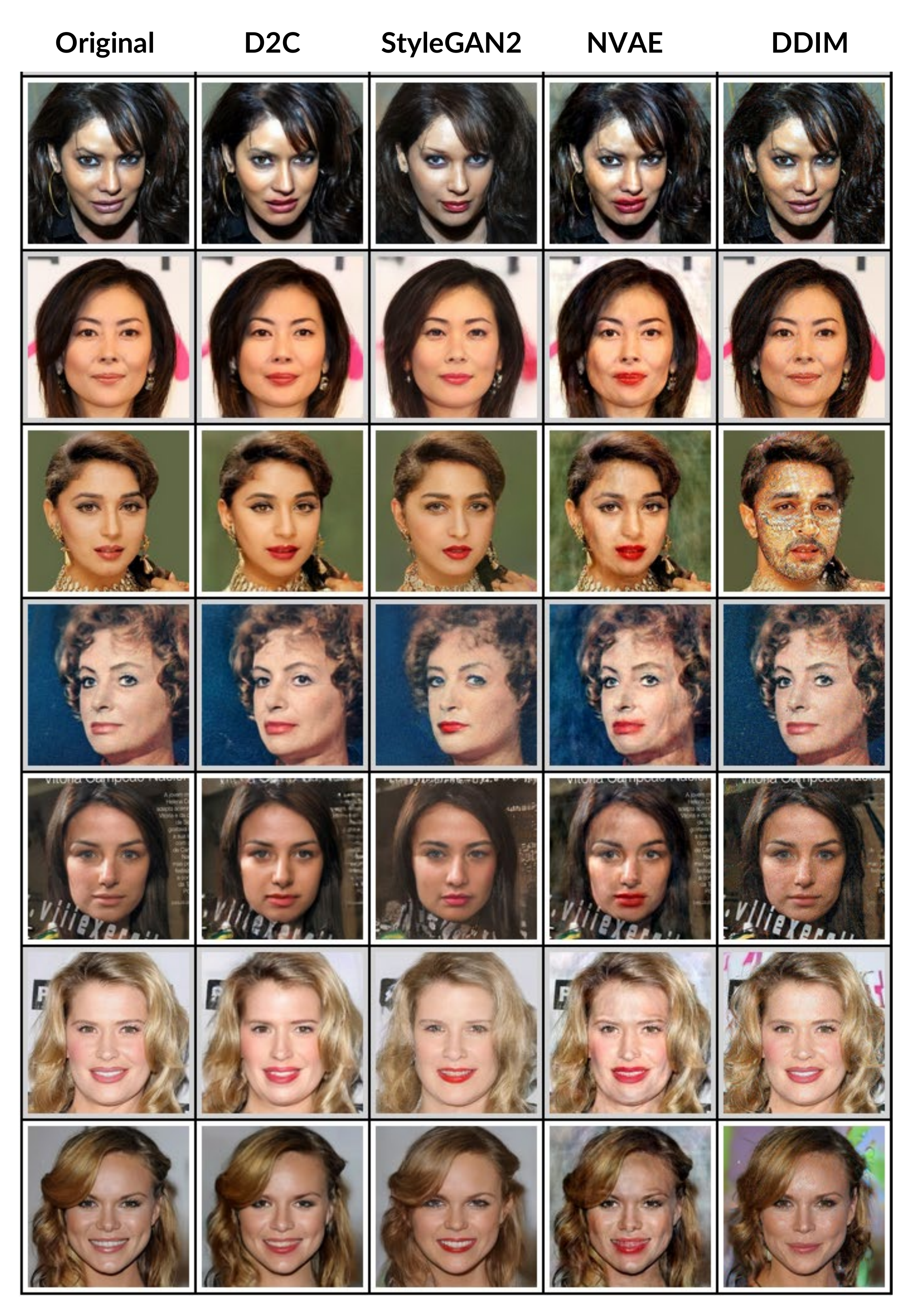}
    \caption{Image manipulation results for \textit{red lipstick}.}
    \label{fig:red-lipstick-large}
\end{figure}

\begin{figure}
    \centering
    \includegraphics[width=\textwidth]{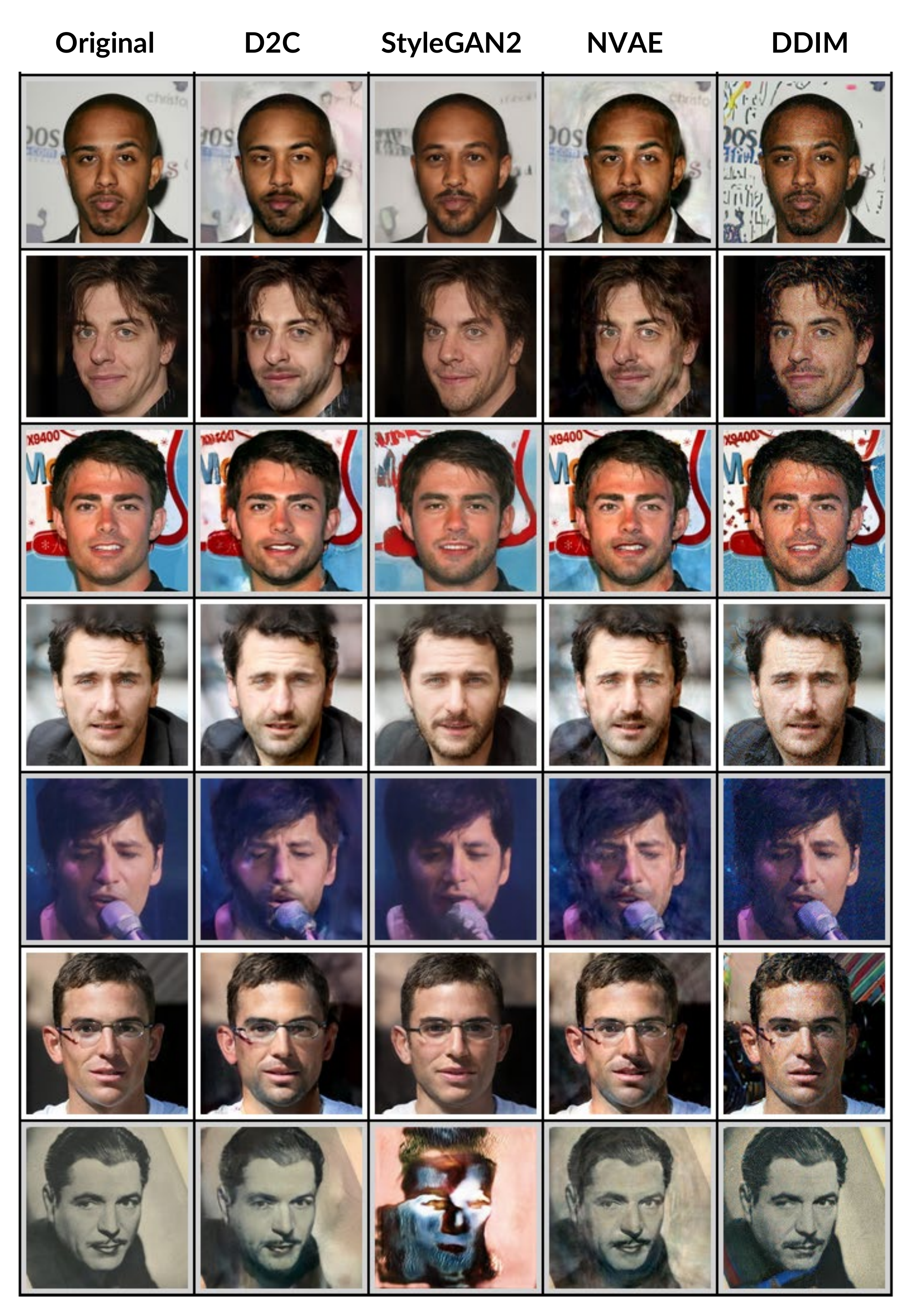}
    \caption{Image manipulation results for \textit{beard}.}
    \label{fig:beard-large}
\end{figure}

\begin{figure}
    \centering
    \includegraphics[width=\textwidth]{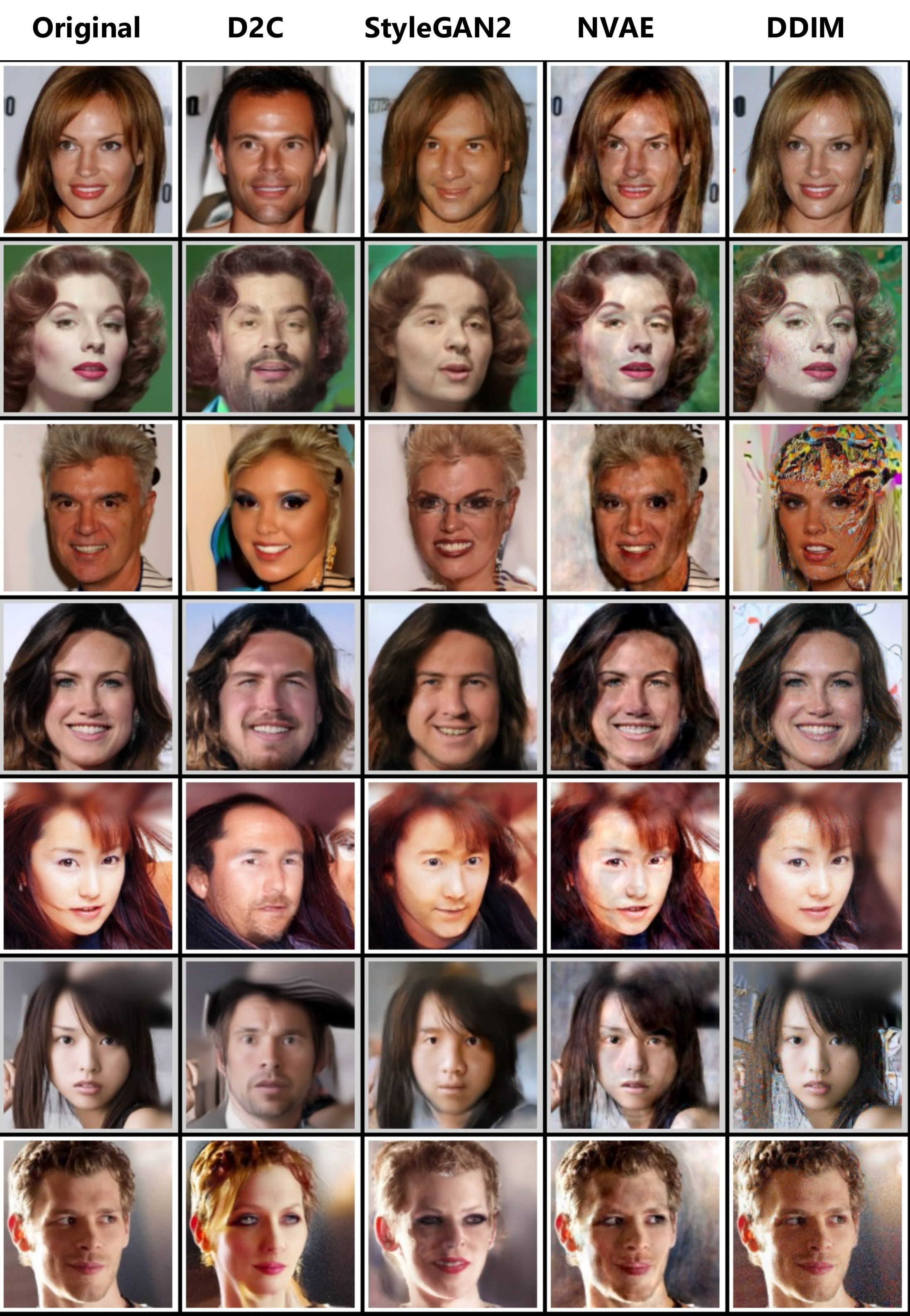}
    \caption{Image manipulation results for \textit{gender}.}
    \label{fig:gender-large}
\end{figure}

\begin{figure*}[!h]
\centering
\begin{subfigure}[b]{0.45\textwidth}
    \includegraphics[width=\textwidth]{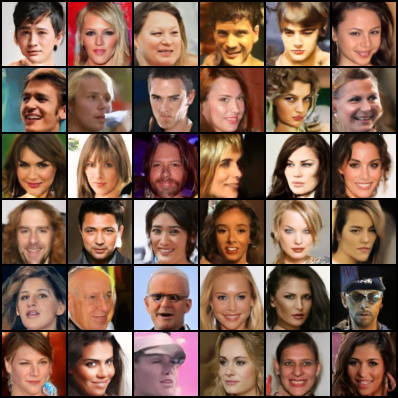}
    \caption{Conditioned on \textit{non-blond} label}
\end{subfigure}%
~
\begin{subfigure}[b]{0.45\textwidth}
    \includegraphics[width=\textwidth]{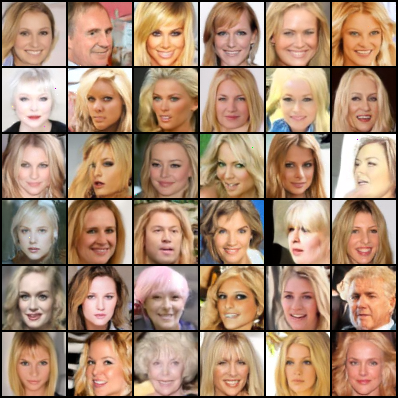}
    \caption{Conditioned on \textit{blond} label}
\end{subfigure}
\caption{Conditional generation with D2C by learning from 100 labeled examples.}
\label{fig:blond_our}
\end{figure*}

\begin{figure*}[!h]
\centering
\begin{subfigure}[b]{0.45\textwidth}
    \includegraphics[width=\textwidth]{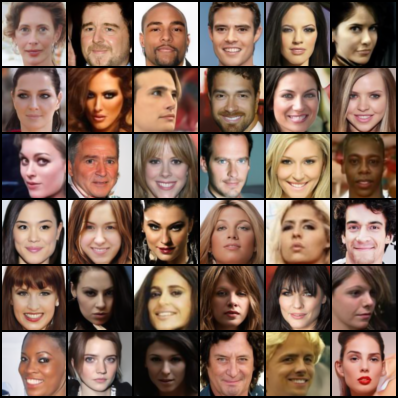}
    \caption{Conditioned on \textit{non-blond} label}
\end{subfigure}%
~
\begin{subfigure}[b]{0.45\textwidth}
    \includegraphics[width=\textwidth]{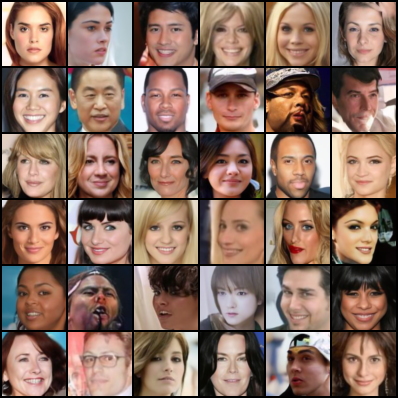}
    \caption{Conditioned on \textit{blond} label}
\end{subfigure}
\caption{Conditional generation with DDIM by learning from 100 labeled examples.}
\label{fig:blond_ddim}
\end{figure*}

\begin{figure*}[!h]
\centering
\begin{subfigure}[b]{0.45\textwidth}
    \includegraphics[width=\textwidth]{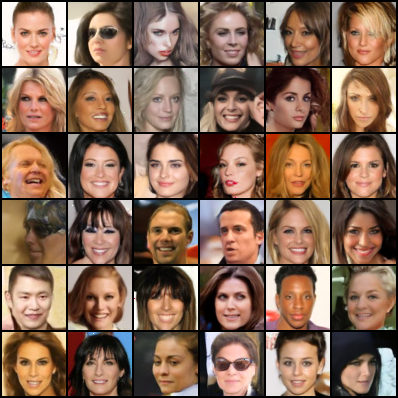}
    \caption{Conditioned on \textit{female} label}
\end{subfigure}%
~
\begin{subfigure}[b]{0.45\textwidth}
    \includegraphics[width=\textwidth]{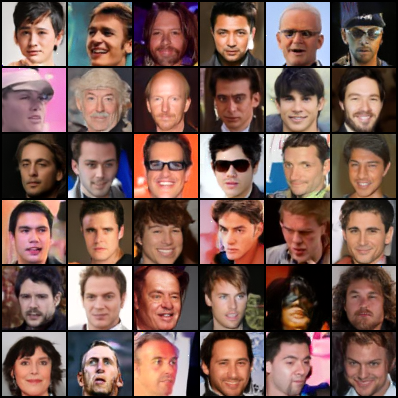}
    \caption{Conditioned on \textit{male} label}
\end{subfigure}
\caption{Conditional generation with D2C by learning from 100 labeled examples.}
\label{fig:gender_our}
\end{figure*}

\begin{figure*}[!h]
\centering
\begin{subfigure}[b]{0.45\textwidth}
    \includegraphics[width=\textwidth]{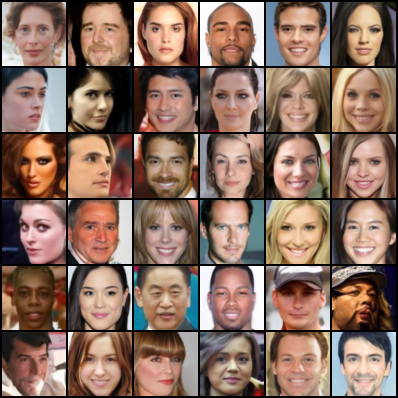}
    \caption{Conditioned on \textit{female} label}
\end{subfigure}%
~
\begin{subfigure}[b]{0.45\textwidth}
    \includegraphics[width=\textwidth]{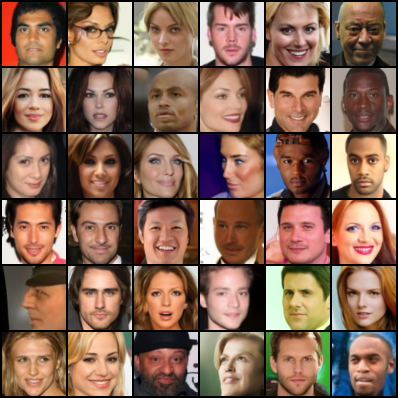}
    \caption{Conditioned on \textit{male} label}
\end{subfigure}
\caption{Conditional generation with DDIM by learning from 100 labeled examples.}
\label{fig:gender_ddim}
\end{figure*}

\section{Broader Impact}
\label{app:impact}

Recent approaches have trained large vision and language models for conditional generation~\cite{ramesh2021zero}.
However, training such models (\textit{e.g.}, text to image generation) would require vast amounts of resources including data, compute and energy. Our work investigate ideas towards reducing the need to provide paired data (\textit{e.g.}, image-text pairs) and instead focus on using unsupervised data.

Since our generative model tries to faithfully reconstruct training images, there is a potential danger that the model will inherit or exacerbate the bias within the data collection process~\cite{song2018learning}. %
Our method also has the risk of being used in unwanted scenarios such as deep fake. Nevertheless, if we are able to monitor and control how the latent variables are used in the downstream task (which may be easier than directly over images, as the latent variables themselves have rich structure), we can better defend against unwanted use of our models by rejecting problematic latent variables before decoding.

\end{document}